\documentclass[twoside,11pt]{article}
\usepackage{journal2e_v1}

\usepackage{url}
\usepackage{epsf}
\usepackage{epsfig}
\usepackage{amsmath}
\usepackage{subfigure}
\usepackage{amssymb}
\usepackage{algorithm}
\usepackage{multirow}
\usepackage{multicol}

\usepackage{graphicx} 

\usepackage{algorithm}
\usepackage{algorithmic}

\def\b{{\bf b}}
\def\C{{\bf C}}

\def\I{{\bf I}}

\def\X{{\bf X}}

\def\s{{\bf s}}

\def\x{{\bf x}}
\def\y{{\bf y}}
\def\z{{\bf z}}

\def\u{{\bf u}}

\def\v{{\bf v}}

\def\w{{\bf w}}
\def\0{{\bf 0}}
\def\1{{\bf 1}}

\def\AM{{\mathcal A}}

\def\RB{{\mathbb R}}
\def\NB{{\mathbb N}}

\def\epsi{\mbox{\boldmath$\epsilon$\unboldmath}}
\def\etab{\mbox{\boldmath$\eta$\unboldmath}}

\def\Tha{\mbox{\boldmath$\Theta$\unboldmath}}

\def\vps{\mbox{\boldmath$\varepsilon$\unboldmath}}

\def\argmin{\mathop{\rm argmin}}
\def\liml{\mathop{\lim}\limits}

\def\sgn{\mathrm{sgn}}

\title{The Bernstein Function: A Unifying Framework of Nonconvex Penalization in Sparse Estimation}

\author{\name Zhihua Zhang  \\
\addr  MOE-Microsoft Key Lab for Intelligent Computing and Intelligent Systems \\
Department of Computer Science and Engineering \\
Shanghai Jiao Tong University \\
800 Dong Chuan Road, Shanghai, China 200240 \\
\texttt{zhihua@sjtu.edu.cn}
}

\date{\today}

\begin{document}

\maketitle

\begin{abstract}%
In this paper we study nonconvex penalization using Bernstein functions.
Since the Bernstein function is concave and nonsmooth at the origin, it can induce a class of nonconvex functions
for high-dimensional sparse estimation problems.
We derive a threshold function based on the Bernstein penalty and give its mathematical properties in sparsity modeling.
We show that a coordinate descent algorithm is especially appropriate for penalized regression problems with the Bernstein penalty.
Additionally, we prove that the Bernstein function can be defined as the concave conjugate of a $\varphi$-divergence and develop a conjugate maximization algorithm for finding the sparse solution.
Finally, we particularly exemplify a family of Bernstein nonconvex penalties based on a generalized Gamma measure and conduct
empirical analysis for this family.
\end{abstract}

\section{Introduction}

Variable  selection  plays a fundamental role in statistical modeling for high-dimensional
data sets, especially when the underlying model has a sparse representation.
The approach based on penalty theory has been widely used for variable selection in the literature.
A principled approach is due to the lasso of \citet{TibshiraniLASSO:1996}, which employs the $\ell_1$-norm penalty and
performs variable selection via the soft threshold
operator. However, \citet{Fan01} pointed out
that the lasso shrinkage method produces
biased estimates for the large coefficients. \citet{ZouJasa:2006} argued that the lasso might not
be an oracle procedure under certain scenarios.

Accordingly,
\citet{Fan01} proposed three criteria for a good penalty function. That is, the resulting estimator should hold
\emph{sparsity}, \emph{continuity} and \emph{unbiasedness}. Moreover, \citet{Fan01} showed that a nonconvex penalty generally
admits the oracle properties.
This leads to recent developments of nonconvex penalization in sparse learning.
There exist many nonconvex penalties, including the $\ell_q$ ($q\in (0,1)$) penalty,  the
smoothly clipped absolute deviation (SCAD) \citep{Fan01},
the minimax concave plus penalty (MCP) \citep{Zhang2010mcp}, the kinetic energy plus  penalty (KEP)~\citep{ZhangTR:2013},  the capped-$\ell_1$ function~\citep{TZhangJMLR:10,ZhangZhang2012},  the nonconvex exponential penalty (EXP)~\citep{BradleyICML:1998,GaoAAAI:2011},
the LOG penalty~\citep{MazumderSparsenet:11,ArmaganDunsonLee},  etc.
These penalties have been demonstrated to have attractive properties theoretically and practically.

On one hand, nonconvex penalty functions typically have the tighter  approximation to the $\ell_0$-norm and hold the oracle properties~\citep{Fan01}.
On the other hand, they would yield computational challenges due to  nondifferentiability and nonconvexity.
Recently,  \citet{MazumderSparsenet:11}
developed  a SparseNet algorithm  base on  coordinate descent.
Especially,  the authors studied the coordinate descent algorithm for the MCP function
\citep[also see][]{BrehenyAAS:2010}. Moreover,  \cite{MazumderSparsenet:11} proposed some desirable properties for threshold operators based on nonconvex
penalties. For example, the threshold operator should be a strict nesting
w.r.t.\ a sparsity parameter.
However,  the authors  claimed that not all nonconvex penalties are suitable for use with  coordinate descent.

In this paper we introduce Bernstein functions into sparse estimation, giving rise to a unifying approach to  nonconvex penalization.
The Bernstein function is
a class of functions whose first-order derivatives are completely monotone~\citep{SSVBook:2010,FellerBook:1971}.
The Bernstein function can be formed as a class of sparsity-inducing nonconvex penalty functions. Moreover,
the Bernstein function has the L\'{e}vy-Khintchine representation.
We particularly exemplify a family of Bernstein nonconvex penalties based on a generalized Gamma measure~\citep{Aalen:1992,Brix:1999}.
The special cases include the KEP,
nonconvex LOG and EXP as well as a
penalty function that we call  linear-fractional (LFR) function. Moreover, we find that the MCP function is
a truncated special version.

The Bernstein function has attractive ability in sparsity modeling. Geometrically, the Bernstein function holds  the property of regular variation~\citep{FellerBook:1971}. That is,   the Bernstein function
bridges the $\ell_{q}$-norm ($0\leq q <1$) and the $\ell_1$-norm. Theoretically,
it admits the oracle properties and can results in a unbiased and continuous sparse estimator. Computationally, the resulting
estimation problem can be efficiently solved by using coordinate descent algorithms. Moreover, the corresponding
threshold operator has to some extend the nesting property~\citep{MazumderSparsenet:11}.

Another important contribution of this paper offers a new construction approach for Bernstein functions. That is,
we show that the Bernstein function can be be defined as concave conjugates of $\varphi$-divergences~\citep{CsiszarPHI:1967,CensorBook:1997}
under certain conditions. This construction illustrates an interesting
connection between LOG and EXP as well as between KEP and MCP~\citep{ZhangNIPS:2012,ZhangTR:2013}.
We note that
\citet{WipfNIPS:2008} used the idea of concave conjugate for expressing the
automatic relevance determination (ARD) cost function, and
\citet{TZhangJMLR:10} derived the bridge penalty by using the idea of concave conjugate.
To the best of our knowledge, however,
our work is the first time to uncover the intrinsic connection between the Bernstein function and the  $\varphi$-divergence.

Based on this new construction approach, we also develop a conjugate-maximization (CM) algorithm for solving  penalized regression problems.
The CM algorithm consists of a C-step and an M-step.
There is an interesting resemblance between CM and EM.
The C-step of CM calculates the concave conjugate of a $\varphi$-divergence with respect to an auxiliary (weight) vector,
while the E-step of EM the expected sufficient statistics with respect to missing data.
The M-steps of both CM and EM are to find the new estimate of the parameter vector in question.
Additionally, the CM algorithm shares the same convergence property with the conventional
EM algorithm~\citep{WuEM:1983}.

It is worth pointing out that the CM algorithm is related to the augmented Lagrangian method~\citep{NocedalWright:2006,CensorBook:1997}.
Additionally,
the CM algorithm enjoys the idea behind
the iterative reweighted $\ell_2$ or $\ell_1$ methods~\citep{ChartrandICASSP:2008,CandesWakinBoyd:2008,WipfNIPS:2008,Daubechies:2010,WipfNagarajan:2010,TZhangJMLR:10}.
Thus, CM also implies a so-called majorization-minimization (MM) procedure~\citep{HunterLiAS:2004}.
An attractive merit of the CM over the existing MM methods is its ability
in handling the choice of tuning parameters, which is a very important issue in nonconvex sparse regularization.

The remainder of this paper is organized as follows.
Section~\ref{sec:lapexp} exploits Bernstein functions in the construction of nonconvex penalties.
In Section~\ref{sec:sest} we investigate sparse estimation problems based on the Bernstein function and devise the coordinate descent algorithm for finding the sparse solution.
In Section~\ref{sec:math} we conduct theoretical analysis of the corresponding sparse estimation problem.
In Section~\ref{sec:cc} we study Bernstein penalty functions based on concave conjugate of the $\varphi$-divergence.
In Section~\ref{sec:cm} we devise the CM algorithm based on the the $\varphi$-divergence.
Finally, we conclude our work in
Section~\ref{sec:conclusion}. 

\section{Nonconvex Penalization via Bernstein Functions}
\label{sec:lapexp}

Suppose we are given a set of training
data $\{(\x_i, y_i): i=1,\ldots, n\}$, where
the $\x_i \in \RB^{p}$ are the input vectors and the $y_i$ are the corresponding
outputs. Moreover, we assume that $\sum_{i=1}^n \x_i=\0$ and $\sum_{i=1}^n y_i=0$. 
We now consider the following linear regression model:
\[
\y = \X \b + \vps,
\]
where $\y=(y_1, \ldots, y_n)^T$ is the $n{\times}1$ output vector,
$\X=[\x_1, \ldots, \x_n]^T$ is the $n {\times} p$
input matrix, and $\vps$ is a Gaussian error vector $N(\vps|\0, \sigma \I_n)$.  We aim to
find a sparse estimate  of regression vector $\b=( b_1, \ldots, b_p)^T$ under the regularization framework.

The classical regularization approach is based on
a penalty function of $\b$. That is,
\[
\min_{\b}  \;  \Big\{F(b) \triangleq   \frac{1}{2} \|\y {-} \X \b\|_2^2 +  P(\b; \lambda) \Big\},
\]
where $P(\cdot)$ is
the regularization term penalizing model complexity and $\lambda$ ($>0$) is the tuning  parameter of balancing the relative significance
of the loss function and the penalty.

A widely used setting for  penalty is $P(\b; \lambda)= \sum_{j=1}^p P(b_j; \lambda)$,
which implies that the penalty function consists of $p$ separable subpenalties.
In order to find a sparse solution of $\b$, one imposes the $\ell_0$-norm penalty $\|\b\|_0$ to  $\b$ (i.e., the number of nonzero elements of $\b$).
However, the resulting optimization problem is usually NP-hard. Alternatively,
the $\ell_1$-norm   $\|\b\|_1=  \sum_{j=1}^p |b_j|$ is an effective convex  penalty.
Recently, some nonconvex alternatives, such as the  log-penalty,
SCAD,  MCP and KEP, have been  employed. Meanwhile, iteratively reweighted $\ell_q$  ($q=1$ or $2$) minimization or coordination descent methods  were developed for finding
sparse solutions.

In this paper we are concerned with nonconvex penalization based on a Bernstein function~\citep{SSVBook:2010}.
Let $f \in C^{\infty}(0, \infty)$ with $f\geq 0$.
We say $f$ is completely monotone if $(-1)^k f^{(k)} \geq 0$ for all $k \in \NB$ and a Bernstein function
if $(-1)^k f^{(k)} \leq 0$ for all $k \in \NB$. It is well known that
$f$ is a Bernstein function if and only if the mapping $s \mapsto \exp(- t f(s))$ is completely monotone for all $t\geq 0$.
Additionally, $f$ is a Bernstein function if and only if it has the representation
\[
f(s) = a + \beta s + \int_{0}^{\infty} {\big[1 - \exp(-  s u) \big]  \nu(d u)} \; \mbox{ for all } s> 0,
\] 
where $a, \beta \geq 0$, and  $\nu$ is the L\'{e}vy measure satisfying additional requirements $\nu(-\infty, 0)=0$ and
$\int_{0}^{\infty} { \min(u, 1) \nu(d u) } < \infty$. Moreover, this representation is
unique. The representation is famous as the L\'{e}vy-Khintchine formula.

Since $\liml_{s\to 0} f(s) =a$ and $\liml_{s\to \infty} \frac{f(s)}{s}=\beta$~\citep{SSVBook:2010},
we will assume that $\liml_{s \to 0}f(s)=0$ and $\liml_{s\to \infty} \frac{f(s)}{s} =0$ to make $a=0$ and $\beta=0$.
Note that  $s^{q}$ for $q \in (0, 1)$ is a Bernstein function of $s$ on $(0, \infty)$ satisfying the above assumptions.
However, $f(s)=s$ is  Bernstein but does not satisfy  the condition $\liml_{s\to \infty} \frac{f(s)}{s} =0$. Indeed,
$f(s)=s$  is an extreme case because  $\beta=1$ and $\nu(d u) =\delta_{0}(u) d u$ (the Dirac Delta measure) in its L\'{e}vy-Khintchine formula.
In fact, the  condition $\liml_{s\to \infty} \frac{f(s)}{s} =0$ aims to exclude this Bernstein function for our concern in this paper.

\subsection{Bernstein Penalty Functions}

We  now define the penalty function $P(\b; \lambda)$ as $\lambda \sum_{j=1}^p \Phi(|b_j|)$,
where the penalty term $\Phi(s)$ is a Bernstein function  of $s$ on $(0, \infty)$ such that $\liml_{s \to 0}\Phi(s)=0$ and $\liml_{s \to \infty} \frac{\Phi(s)}{s} =0$.
Clearly,
$\Phi(s)$
is nonnegative, nondecreasing and concave on $(0, \infty)$, because $\Phi(s)\geq 0$, $\Phi'(s)\geq 0$ and $\Phi{''}(s)\leq 0$.
Moreover, we have the following theorem.

\begin{theorem} \label{thm:lapexp00}
Let $\Phi
(s)$ be a nonzero Bernstein function of $s$ on $(0, \infty)$. Assume $\liml_{s \to 0}\Phi(s)=0$ and $\liml_{s\to \infty} \frac{\Phi(s)}{s} =0$.
Then
\begin{enumerate}
\item[\emph{(a)}] $\Phi(|b|)$ is a nonnegative and nonconvex function of $b$ on $(-\infty, \infty)$, and an increasing function
of $|b|$ on  $[0, \infty)$.
\item[\emph{(b)}] $\Phi(|b|)$ is continuous  w.r.t.\ $b$ but nondifferentiable at the origin.
\end{enumerate}
\end{theorem}

Recall that under the conditions in Theorem~\ref{thm:lapexp00}, $a$ and $\beta$ in the L\'{e}vy-Khintchine formula vanish.
Theorem~\ref{thm:lapexp00} (b) shows that $\Phi'(|b|)$  is singular at the origin. Thus,   $\Phi(|b|)$  can define a
class of sparsity-inducing nonconvex penalty functions.
We can clearly see the connection of the bridge penalty $|b|^{\rho}$ with the $\ell_0$-norm and the $\ell_1$-norm
as $\rho$ goas from $0$ to 1. However, the sparse estimator resulted from the bridge penalty is not continuous.
This would make numerical computations and model predictions unstable~\citep{Fan01}.
In this paper we consider another class of Bernstein nonconvex penalties.

In particular, to explore the relationship of the Bernstein penalties with the $\ell_0$-norm and the $\ell_1$-norm,
we further assume that
$\liml_{s \to 0} \Phi'(s)<\infty$.
Since $\Phi(s)$ is a nonzero Bernstein function of $s$, we can conclude that $\Phi'(0)>0$.
If it is not true, we have $\Phi'(s)=0$ due to $\Phi'(s) \leq \Phi'(0)$.
This implies that $\Phi(s)=0$ for any $s \in (0, \infty)$
because $\Phi(0)=0$. This conflicts with that   $\Phi(s)$ is  nonzero.
Similarly, we can also deduce $\Phi{''}(0) < 0$. Based on this fact, we can change the assumption
$\Phi'(0)<\infty$ as
$\Phi'(0)=1$  without loss of generality. In fact, we can replace $\Phi(s)$ with $\frac{\Phi(s)}{\Phi'(0)}$ to met this assumption, because
the resulting $\Phi(s)$ is still Bernstein and satisfies $\Phi(0)=0$, $\liml_{s \to \infty} \frac{\Phi(s)}{s} =0$
and $\Phi'(0) = 1$.

\begin{theorem} \label{thm:lp2}
Assume the conditions in Theorem~\ref{thm:lapexp00} hold. If
$\Phi'(0) = \liml_{s \to 0} \Phi'(s)= 1$, then
\[
\lim_{\alpha \to 0+} \frac{\Phi(\alpha |b|)}{\Phi(\alpha)} = |b|.
\]
Furthermore, if  $\liml_{s \to \infty} \frac{{s \Phi'(s)}}{\Phi(s)}$ exists,
then for $b\neq 0$,
\[
\lim_{\alpha \to \infty} \frac{\Phi(\alpha |b|)}{\Phi(\alpha)} = |b|^{\gamma},
\]
where $\gamma=\liml_{s \to \infty} \frac{{s \Phi'(s)}}{\Phi(s)} \in [0, 1)$.
Especially, if $\gamma \in (0, 1)$, we also have
\[
 \lim_{\alpha \to \infty} \frac{ \Phi'(\alpha |b|)}{\Phi'(\alpha)} =|b|^{\gamma{-}1}.
\]
\end{theorem}

\paragraph{Remarks 1} \; It is worth noting that  $\Phi'(s)$ is completely monotone on $(0, \infty)$.
Moreover,  $\Phi'(s)$ is the Laplace transform of some probability distribution due to $\Phi'(0)=1$~\citep{FellerBook:1971}.
Additionally, Lemma~\ref{lem:lapexp} (see the appendix) shows that $\liml_{s \to \infty} \frac{{s \Phi'(s)}}{\Phi(s)}= 0$
whenever $\liml_{s \to \infty} \Phi(s)< \infty$. If $\liml_{s \to \infty} \Phi(s)= \infty$, we take $\Psi(s) \triangleq \log(1+\phi(s))$
which is also Bernstein and holds the conditions $\Psi(0)=0$,  $\Psi'(0)=1$ and $\Psi'(\infty)=0$.  In this case,
consider $\Psi'(s) = \frac{\Phi'(s)}{1+ \Phi(s)}$ and $\liml_{s \to \infty} \frac{{s \Phi'(s)}}{\Phi(s)}=  \liml_{s \to \infty} s \Psi'(s)$.
Thus, Lemma~~\ref{lem:lapexp}-(b)  directly applies the Bernstein function $\Psi(s)$. In summary, the condition ``$\liml_{s \to \infty} \frac{{s \Phi'(s)}}{\Phi(s)}$ exists" is essentially natural.

\paragraph{Remarks 2} \;
It follows from Theorem~1 in Chapter VIII.9 of \cite{FellerBook:1971} that $\liml_{s \to \infty} \frac{{s \Phi'(s)}}{\Phi(s)}= \gamma \in (0, 1)$ if and only if
$ \liml_{\alpha \to \infty} \frac{ \Phi'(\alpha |b|)}{\Phi'(\alpha)} =|b|^{\gamma{-}1}$.
However, $\liml_{\alpha \to \infty} \frac{ \Phi'(\alpha |b|)}{\Phi'(\alpha)} =|b|^{{-}1}$ (i.e., $\gamma=0$) is only sufficient for $\liml_{s \to \infty} \frac{{s \Phi'(s)}}{\Phi(s)}=0$.
It is also seen from Lemma~\ref{lem:01} in the appendix that  $\liml_{s \to \infty} \frac{\Phi(s)}{\log (s)}< \infty$ is a sufficient condition for $\liml_{s \to \infty} \frac{{s \Phi'(s)}}{\Phi(s)}=0$ and from Lemma~\ref{lem:03} in the appendix that $\liml_{s \to \infty} \frac{{s \Phi'(s)}}{\Phi(s)}=\gamma \in [0, 1)$.

The second part of Theorem~\ref{thm:lp2} shows that the property of regular variation for the Bernstein function $\Phi(s)$ and its
derivative $\Phi'(s)$~\citep{FellerBook:1971}. That is, $\Phi(s)$ and $\Phi'(s)$ vary regularly with exponents  $\gamma$ and $\gamma{-}1$,
respectively. If  $\liml_{s \to \infty} \frac{\Phi(s)}{\log (s)}<\infty$, then $\Phi(s)$  varies slowly (i.e., $\gamma=0$).
This property
implies an important connection of the Bernstein function with the $\ell_0$-norm and $\ell_1$-norm. With this connection,
we see that $\alpha$ plays a role of sparsity parameter because it measures sparseness of $\Phi(\alpha |\b|)/\Phi(\alpha)$.
In the following we present a  family of
Bernstein functions which admit the properties in Theorem~\ref{thm:lp2}.

\begin{table}[!ht]
\begin{center}
\caption{Several Bernstein  functions $\Phi_{\rho}(s)$ on $[0, \infty)$ as well as their  derivatives} \label{tab:exam}
\begin{tabular}{llll}
  \hline
 & Bernstein functions & First-order derivatives &  L\'{e}vy measures \\ \hline
KEP & $\Phi_{-1}(s) =  \sqrt{2 s {+}1} -1 $ & $\Phi_{-1}'(s)=  \frac{1}{\sqrt{2  s +1}}$
& $\nu(du) = \frac{1}{ \sqrt{2\pi}} u^{{-}\frac{3}{2}}  \exp({-} \frac{u}{2}) d u$\\
LOG &  $\Phi_{0}(s)=  \log\big(s  {+}1 \big)$ & $\Phi_{0}'(s)=  \frac{1}{s  {+}1}$ &
  $\nu(du) = \frac{1}{ u} \exp( {-} {u}) d u$ \\
LFR & $\Phi_{{1}/{2}}(s) =  \frac{2 s}{ s +2}$   & $\Phi'_{{1}/{2}}(s) =  \frac{4 }{({s}+{2})^2}$
&  $\nu(d u) = {4} \exp(- {2 u} ) d u$ \\
EXP & $\Phi_{1}(s) =  1- \exp(-  s)$ & $\Phi_{1}'(s) =   \exp(- s)$ & $\nu(d u) =  \delta_{1}(u) d u$ \\
\hline
\end{tabular}
\end{center}
\end{table}

\subsection{Examples}

We consider a family of
Bernstein functions of the form
\begin{equation} \label{eqn:first}
\Phi_{\rho}(s) = \left\{\begin{array}{ll} \log(1+ s) & \mbox{if } \rho=0, \\
  \frac{1}{\rho}\Big[1- \big(1+ {(1{-}\rho)} s \big)^{-\frac{\rho}{1-\rho}} \Big]  & \mbox{if } \rho <1 \mbox{ and } \rho\neq 0, \\  1- \exp(- s) & \mbox{if } \rho=1. \end{array}
  \right.
\end{equation}
It can be directly verified that $\Phi_{0}(s) = \liml_{\rho\to 0} \Phi_{\rho}(s)$ and $\Phi_{1}(s) = \liml_{\rho\to 1-} \Phi_{\rho}(s)$.
The corresponding L\'{e}vy measure is
\begin{equation} \label{eqn:first_nu}
\nu(d u) =  \frac{((1{-}\rho))^{-1/(1{-}\rho)}}{\Gamma(1/(1{-}\rho))} u^{\frac{\rho}{1-\rho}-1} \exp\Big( {-} \frac{u}{(1{-}\rho) } \Big) d u.
\end{equation}
Note that ${u}  \nu(d u)$ forms a Gamma measure for random variable $u$. Thus,
this L\'{e}vy measure $\nu(d u)$ is referred to as a generalized Gamma measure~\citep{Brix:1999}.
This family of the Bernstein functions were studied by \cite{Aalen:1992} for survival analysis.
We here show that they can be also used for sparsity modeling.

It is easily seen that the Bernstein functions $\Phi_{\rho}(s)$ for $\rho \leq  1$ satisfy the conditions:
$\Phi(0)=0$,
$\Phi'(0) =1$ and $(-1)^{k}\Phi^{(k+1)}(0)< \infty$ for $k \in \NB$, in Theorem \ref{thm:lp2} and Lemma~\ref{lem:lapexp} (see the appendix). Thus,
$\Phi_{\rho}(s)$ for  $\rho  \leq  1$ have the properties given in Theorem \ref{thm:lp2} and Lemma~\ref{lem:lapexp}.
These properties show that when letting $s=|b|$, the Bernstein functions $\Phi(|b|)$ form nonconvex penalties.

The derivative of $\Phi_{\rho}(s)$ is defined by
\begin{equation} \label{eqn:deriv}
\Phi'_{\rho}(s) = \left\{\begin{array}{ll} \frac{1}{1 {+} s} & \mbox{if } \rho=0, \\
\big(1+ {(1{-}\rho)} s \big)^{-\frac{1}{1-\rho}}  & \mbox{if } \rho <1 \mbox{ and } \rho\neq 0, \\  \exp(- s) & \mbox{if } \rho=1. \end{array}
  \right.
\end{equation}
It  is also directly verified that $\Phi'_{0}(s) = \liml_{\rho\to 0} \Phi'_{\rho}(s)$ and $\Phi'_{1}(s) = \liml_{\rho\to 1-} \Phi'_{\rho}(s)$.
When $\rho \in [0, 1]$, we have $\liml_{s \to \infty} \frac{s \Phi'_{\rho}(s)}{\Phi_{\rho}(s)} =0$ (or $\liml_{s\rightarrow \infty} \frac{\Phi(s)}{\log(s)}<\infty$).
When  $\rho < 0$, we then have $\liml_{s \to \infty} \frac{s \Phi'_{\rho}(s)}{\Phi_{\rho}(s)} = \frac{\rho}{\rho{-}1} \in (0, 1)$.

\begin{proposition}\label{pro:33} Let $\Phi_{\rho}(s)$ on $(0, \infty)$ be defined in (\ref{eqn:first}). Then
\begin{enumerate}
\item[\emph{(a)}] If $-\infty<\rho_1< \rho_2\leq 1$ then $\Phi'_{\rho_1}(s)\geq \Phi'_{\rho_2}(s)$ and $\Phi_{\rho_1}(s)\geq \Phi_{\rho_2}(s)$;
\item[\emph{(b)}] $\liml_{\alpha \to \infty} \frac{ \Phi'_{\rho}( \alpha)}{\alpha^{\gamma-1}} = (1-\gamma)^{1-\gamma}$ where $\gamma=0$ if $\rho \in (0, 1]$ and $\gamma=\frac{\rho}{\rho{-}1}$ if $\rho \in (-\infty, 0]$, and
\[
\liml_{\alpha \to \infty} \frac{\Phi_{\rho}(\alpha s)}{\Phi_{\rho}(\alpha )}= \left\{\begin{array}{ll} 1 & \mbox{if } \rho \in [0, 1], \\
s^{\frac{\rho}{\rho{-}1}} & \mbox{if } \rho \in (-\infty, 0). \end{array} \right.
\]
\end{enumerate}
\end{proposition}

Proposition~\ref{pro:33}-(b) shows the property of regular variation for $\Phi_{\rho}(s)$; that is, $\Phi_{\rho}(s)$ varies slowly when $0 \leq \rho \leq 1$, while it
varies regularly with exponent $\rho/(\rho{-}1)$ when $\rho < 0$.
Thus, $\frac{\Phi_{\rho}(\alpha |b|)}{\Phi_{\rho}(\alpha )}$ for $\rho < 0$ approaches to the $\ell_{\rho/(\rho{-}1)}$-norm $\|b\|_{\rho/(\rho{-}1)}$
as $\alpha \to \infty$.

We list four special Bernstein functions in Table~\ref{tab:exam} by taking different $\rho$. Specifically,
these  penalties are  the kinetic energy plus (KEP) function, nonconvex \emph{log-penalty} (LOG),
nonconvex \emph{exponential-penalty} (EXP), and \emph{linear-fractional} (LFR) function, respectively. Figure~\ref{fig:berns} depicts these functions and their derivatives.
In Table~\ref{tab:exam} we also give the L\'{e}vy measures  corresponding to these functions.
Clearly, KEP gets a continuum of penalties from $\ell_{1/2}$ to the $\ell_1$, as varying $\alpha$
from $\infty$ to $0$~\citep{ZhangTR:2013}. But the LOG, EXP and LFR penalties get the entire continuum of penalties from $\ell_{0}$ to the $\ell_1$.
The LOG, EXP and LFR penalties  have been applied in the literature~\citep{BradleyICML:1998,GaoAAAI:2011,WestonJMLR:2003,GemanPAMI:1992,NikolovaSIAM:2005}.
In image processing and computer vision, these functions are usually also called \emph{potential functions}.
However, to the best of our  knowledge, there is no work to establish their connection with Bernstein functions.

\begin{figure}[!ht]
\centering
\subfigure[$\Phi_{\rho}(s)$ ]{\includegraphics[width=75mm,height=50mm]{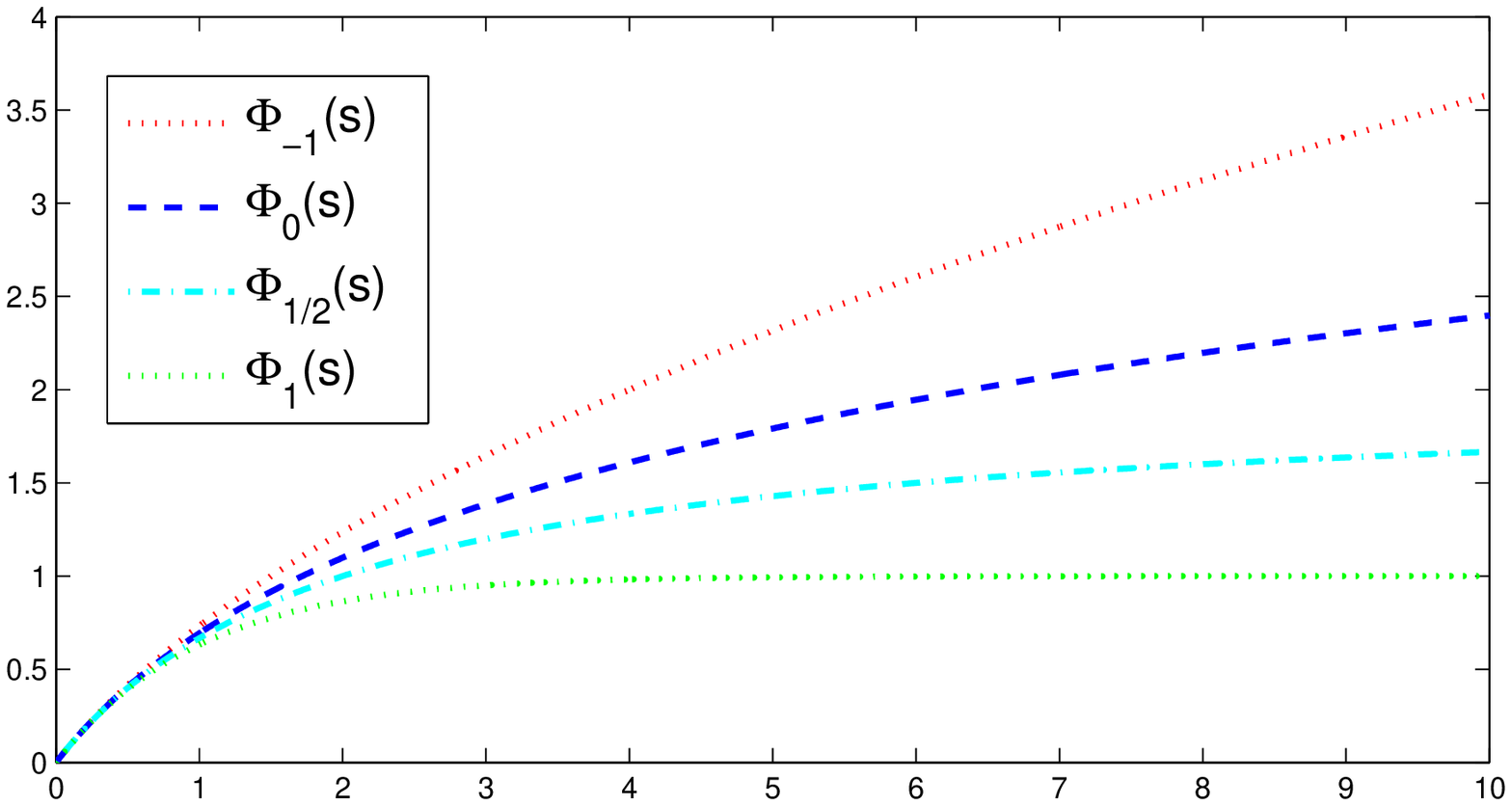}} \subfigure[$\Phi'_{\rho}(s)$]{\includegraphics[width=75mm,height=50mm]{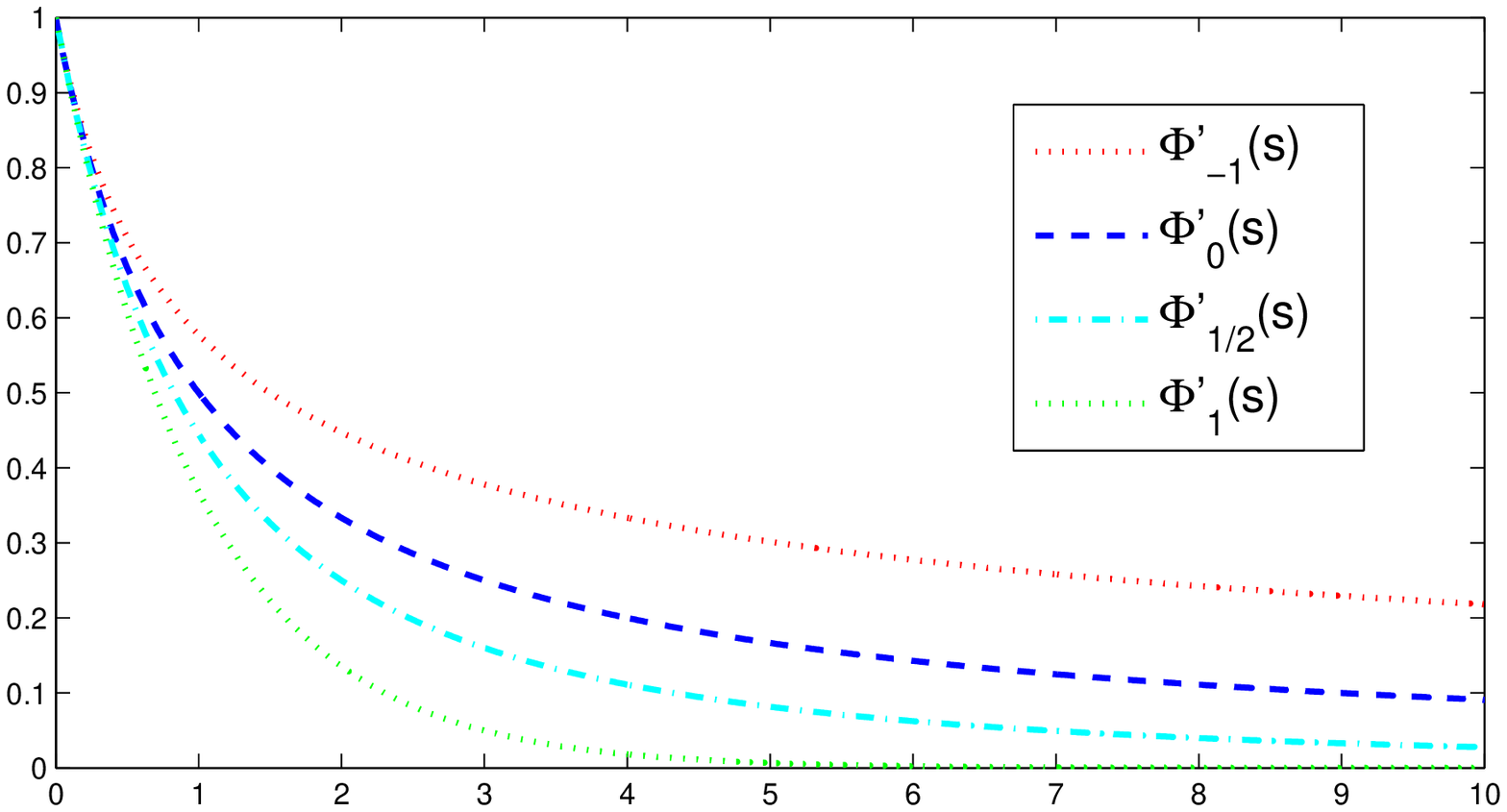}} \\
\caption{(a) The Bernstein functions $\Phi_{\rho}(s)$ for $\rho=-1$, $\rho=0$, $\rho=\frac{1}{2}$ and $\rho=1$ corresponding to KEP, LOG, LFR and EXP.  (b) The corresponding derivatives $\Phi'_{\rho}(s)$.}
\label{fig:berns}
\end{figure}

\begin{figure}[!ht]
\centering
 \subfigure[]{\includegraphics[width=100mm,height=70mm]{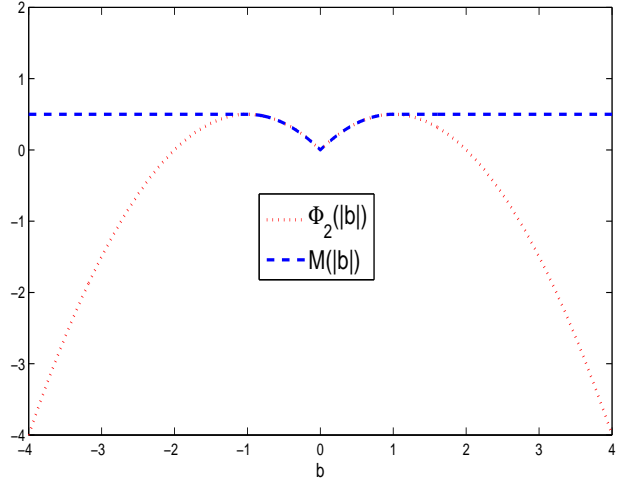}} \\
\caption{The Bernstein function $\Phi_{2}(|b|)$ and the MCP function $M(|b|)$.}
\label{fig:berns2}
\end{figure}

Finally, we note that the MCP function can be regarded as a truncated version of  $\Phi_{2}(s)$ (i.e.,  $\rho=2$).
Clearly,  $\Phi_{2}(s)$ is well-defined for $s\geq 0$ but no longer Bernstein, because $\Phi_{2}(s)$ is negative when $s> 2$.
Moreover,
it is decreasing when $s\geq 1$ (see Figure~\ref{fig:berns2}).  To make a concave penalty function from $\Phi_{2}(s)$, we truncate
$\Phi_{2}(s)$ as $1/2$ whenever $s\geq 1$,  yielding the MCP function. That is,
\begin{equation} \label{eqn:mcp}
 M(\alpha s) = \left\{\begin{array}{ll} \frac{1}{2} & \mbox{ if } s\geq \frac{1}{\alpha}, \\
\alpha s - \frac{\alpha^2 s^2}{2} & \mbox{ if } s < \frac{1}{\alpha}.  \end{array} \right.
\end{equation}

\section{Sparse Estimation Based on Bernstein Penalty Functions}
\label{sec:sest}

We now study  mathematical  properties of  the sparse estimators based on Bernstein penalty functions.
These properties show that Bernstein penalty functions are suitable for use of a coordinate descent algorithm~\citep{MazumderSparsenet:11}.

\subsection{Threshold Operators}
\label{sec:threshold}

Let $\Phi(|b|)$ be a Bernstein penalty function.
Following  \cite{Fan01}, we  define the univariate penalized least squares problem
\begin{equation} \label{eqn:general}
J_{1}(b)\triangleq \frac{1}{2} (z- b)^2 +  {\lambda} \Phi(|b|),
\end{equation}
where $z=\x^T\y$.  \citet{Fan01} stated that a good penalty should result in an estimator with three properties.
(a) ``Unbiasedness:" it is nearly unbiased when the true unknown parameter is large; (b) ``Sparsity:" it is a threshold rule, which
automatically sets small estimated coefficients to zero; (c) ``Continuity:" it is continuous in $z$ to avoid instability
in model computation and prediction.

According to the discussion in \citet{Fan01}, the resulting
estimator from (\ref{eqn:general}) is nearly unbiased if $\Phi'(|b|) \rightarrow 0$ as $|b|\rightarrow \infty$.
The Bernstein penalty function satisfies the conditions
$\Phi(0+)=0$ and
$\liml_{s \to \infty} \Phi'(s) = 0$, so it can result in an unbiased sparse estimator.

\begin{theorem} \label{thm:sparsty} Let $\Phi(s)$ be a nonzero Bernstein function of $s$ on $(0, \infty)$ such that
$\Phi(0)= 0$  and $\liml_{s \to \infty} \frac{\Phi(s)}{s}=\liml_{s \to \infty} \Phi'(s) = 0$.
Consider the penalized least squares problem in (\ref{eqn:general}).
\begin{enumerate}
\item[\emph{(i)}] If $\lambda \leq -\frac{1}{\Phi{''}(0)}$, then the resulting estimator  is defined as
\[
\hat{b} = S(z, \lambda) \triangleq  \left\{ \begin{array}{ll}
{{\sgn}(z)} \kappa(|z|) & \textrm{ if } |z| >
{\lambda} \Phi{'}(0), \\
0 & \textrm{ if } |z| \leq {\lambda}\Phi{'}(0),
\end{array} \right.
\]
where $\kappa(|z|)\in (0, |z|)$ is the unique positive root of
$b {+} {\lambda} \Phi'(b) {-} |z|=0$ in $b$.
\item[\emph{(ii)}] If $\lambda > -\frac{1}{\Phi{''}(0)}$, then the resulting estimator  is defined as
\[
\hat{b} = S(z, \lambda) \triangleq  \left\{ \begin{array}{ll}
{{\sgn}(z)} \kappa(|z|) & \textrm{ if } |z| >
s^* + {\lambda} \Phi'(s^*), \\
0 & \textrm{ if } |z|\leq s^* + {\lambda} \Phi'(s^*),
\end{array} \right.
\]
where $s^*>0$ is the unique root of $1 {+} {\lambda} \Phi{''}(s)=0$ and $\kappa(|z|)$
is the unique root of $b {+} {\lambda} \Phi'(b) {-} |z|=0$ on $(s^*, |z|)$.
\end{enumerate}
\end{theorem}

As we see earlier, we always have $\Phi'(0)>0$ and $\Phi{''}(0)<0$.
It is worth noting that when $\lambda \leq -\frac{1}{\Phi{''}(0)}$ the function $h(b)\triangleq b+ {\lambda} \Phi'(b) - |z|$
is increasing on $(0, |z|)$ and that when $\lambda > -\frac{1}{\Phi{''}(0)}$ it
is also increasing on $(s^*, |z|)$. Thus, we can employ the bisection method
to find the corresponding root $\kappa(|z|)$. 
We will see that an analytic solution for $\kappa(|z|)$
is available when $\Phi(s)$ is either of LOG and LFR. Therefore,
a coordinate descent algorithm is especially appropriate for
Bernstein penalty functions, which will be presented in Section~\ref{sec:cda}.

As stated by  \citet{Fan01}, it suffices for the resulting estimator to be a threshold rule  that the minimum of
the function $|b|+ \lambda \Phi'(|b|)$ is positive. Moreover, a sufficient and necessary condition for ``continuity" is the
the minimum of
$|b|+ \lambda \Phi'(|b|)$ is attained at $0$. In our case, it follows from the proof of Theorem \ref{thm:sparsty}
that when $\lambda \leq -\frac{1}{\Phi{''}(0)}$, $|b|+ \lambda \Phi'(|b|)$ attains its minimum value ${\lambda}\Phi'(0)$ at $s^{*}=0$. Thus,
the resulting estimator  is sparse and continuous when $\lambda \leq -\frac{1}{\Phi{''}(0)}$.
In fact, the continuity can be also concluded directly from Theorem~\ref{thm:sparsty}-(i).
Specifically, when  $\lambda \leq -\frac{1}{\Phi{''}(0)}$,
we have $\kappa(\lambda \Phi'(0)) =0$ because  $0$ is the unique root of equation $b {+} {\lambda} \Phi'(b) {-} {\lambda} \Phi'(0)=0$.

Recall that if $\Phi(s)=s^{q}$ with $q \in (0, 1)$, we have $\liml_{s \to 0} \Phi'(s)=+\infty$  and
$\liml_{s \to 0} \Phi{''}(s)=-\infty$. This implies that $\lambda \leq -\frac{1}{\Phi{''}(0)}$
does  not hold. In other words, this  penalty  cannot result in a continuous solution.

In this paper we are especially concerned with the Bernstein penalty functions which satisfy the conditions
in Theorem~\ref{thm:lp2}.
In this case, since $-\infty<\Phi{''}(0)<0$ and $0<\Phi'(0)<\infty$,
such Bernstein penalties are able to result in a continuous sparse
solution. Consider the regular variation property of $\Phi(s)$ given in Theorem~\ref{thm:lp2}. We
let  $P(b; \lambda) =\lambda {\Phi(\alpha |b|)}$ and $\lambda=\frac{\eta}{\Phi(\alpha)}$
where  $\eta$ and $\alpha$ are positive constants.
We now denote the threshold operator  $S(z, \lambda)$ in Theorem~\ref{thm:sparsty} by $S_{\alpha}(z, \eta)$. As a direct corollary of Theorem~\ref{thm:sparsty}, we particularly have the following results.

\begin{corollary} \label{cor:threshold} Assume $\Phi'(0) =1$ and $\Phi{''}(0)>-\infty$. Let $P(b; \lambda) =\lambda {\Phi(\alpha |b|)}$ and $\lambda=\frac{\eta}{\Phi(\alpha)}$ where $\alpha>0$ and $\eta>0$,
and let $S_{\alpha}(z, \eta)$ be the threshold operator defined in Theorem~\ref{thm:sparsty}.
\begin{enumerate}
\item[\emph{(i)}] If $\eta \leq -\frac{\Phi(\alpha)}{\alpha^2 \Phi{''}(0)}$, then the resulting estimator  is defined as
\[
\hat{b} = S_{\alpha}(z, \eta) \triangleq  \left\{ \begin{array}{ll}
{{\sgn}(z)} \kappa(|z|) & \textrm{ if } |z| >
\frac{\alpha}{\Phi(\alpha)} \eta, \\
0 & \textrm{ if } |z| \leq \frac{\alpha}{\Phi(\alpha)} \eta,
\end{array} \right.
\]
where $\kappa(|z|)\in (0, |z|)$ is the unique positive root of
$b+ \frac{\eta \alpha}{\Phi(\alpha)} \Phi'(\alpha b) - |z|=0$ w.r.t.  $b$.
\item[\emph{(ii)}] If $\eta > - \frac{\Phi(\alpha)}{\alpha^2 \Phi{''}(0)}$, then the resulting estimator  is defined as
\[
\hat{b} = S_{\alpha}(z, \eta)  \triangleq  \left\{ \begin{array}{ll}
{{\sgn}(z)} \kappa(|z|) & \textrm{ if } |z| >
s^* + \frac{\alpha \Phi'(\alpha s^*)}{\Phi(\alpha)}  \eta, \\
0 & \textrm{ if } |z|\leq s^* + \frac{\alpha \Phi'(\alpha s^*)}{\Phi(\alpha)}  \eta,
\end{array} \right.
\]
where $s^*>0$ is the unique root of $1+ \frac{\eta  \alpha^2}{\Phi(\alpha)} \Phi{''}(\alpha s)=0$ and $\kappa(|z|)$ is the unique root of the equation $b+ \frac{\eta \alpha}{\Phi(\alpha)}   \Phi'(\alpha b) - |z|=0$ on $(s^*, |z|)$.
\end{enumerate}
\end{corollary}

\begin{proposition} \label{thm:nest} Assume  $\Phi'(0) =1$ and $\Phi{''}(0)>-\infty$.
Then
\begin{enumerate}
\item[\emph{(a)}]  $\frac{\alpha}{\Phi(\alpha)}> 1$, $\frac{\alpha}{\Phi(\alpha)}$ is increasing  and $\frac{1}{\Phi(\alpha)}$ is decreasing both in $\alpha$ on $(0, \infty)$. Moreover, $\liml_{\alpha \to 0+} \frac{\alpha }{\Phi(\alpha)}=1$ and $\liml_{\alpha \to \infty} \frac{\alpha }{\Phi(\alpha)} = \infty$.
\item[\emph{(b)}] The root $\kappa(|z|)$ is strictly   increasing w.r.t.\ $|z|$.
\end{enumerate}
\end{proposition}

The Bernstein function $\Phi_{\rho}$ given in (\ref{eqn:first}) satisfies the conditions in Corollary~\ref{cor:threshold} and Proposition~\ref{thm:nest}.
Recall that $\alpha$ controls sparseness of $\Phi(\alpha|b|)/\Phi(\alpha)$ as it increases from 0 to $\infty$.
It follows from Proposition~\ref{thm:nest} that $|z|\geq \eta$ due to $|z|\geq \frac{\eta \alpha}{\Phi(\alpha)}$. This implies that
the Bernstein function $\Phi(\alpha|b|)/\Phi(\alpha)$
has stronger sparseness  than the $\ell_1$-norm when $\eta \leq -\frac{\Phi(\alpha)}{\alpha^2 \Phi{''}(0)}$. Moreover, for a fixed $\eta$,
there is a strict nesting of
the shrinkage threshold $\frac{\eta \alpha}{\Phi(\alpha)}$  as $\alpha$ increases.   Thus,
the Bernstein penalty to some extent satisfies the nesting property, a
desirable property for threshold functions pointed out by~\citet{MazumderSparsenet:11}.

As we stated earlier,
when $\rho \in[0, 1]$ $\Phi_{\rho}$ bridges the $\ell_0$-norm and the $\ell_1$-norm.
We now explore a connection of the threshold operator  $S_{\alpha}(z, \eta)$ with the soft threshold operator
based on the lasso
and the hard threshold operator based on the $\ell_0$-norm.

\begin{theorem} \label{thm:0-1}
Let  $S_{\alpha}(z, \eta)$ be the threshold operator defined in Corollary~\ref{cor:threshold}. Then
\[
\lim_{\alpha \to 0+}  S_{\alpha}(z, \eta) = \left\{ \begin{array}{ll}
{\sgn}(z) (|z| - \eta) & \textrm{ if } |z| >
{\eta}, \\
0 & \textrm{ if } |z| \leq {\eta}.
\end{array} \right.
\]
Furthermore, if $\liml_{\alpha \to \infty}  \frac{\alpha \Phi'(\alpha)} {\Phi(\alpha)} =0$  or
$\liml_{\alpha \to \infty} \frac{\Phi(\alpha)}{\log(\alpha)} < \infty$,
then
\[
\lim_{\alpha \to \infty}  S_{\alpha}(z, \eta) = \left\{ \begin{array}{ll}
z & \textrm{ if } |z| >
0, \\
0 & \textrm{ if } |z| \leq 0.
\end{array} \right.
\]
\end{theorem}

In the limiting case of $\alpha \to 0$, Theorem~\ref{thm:0-1} shows that the threshold function $S_{\alpha}(z, \eta)$
approaches the soft threshold function
$\sgn(z) (|z|-\eta)_{+}$. However,
as $\alpha \to \infty$, the limiting  solution does not fully agree with the hard threshold function, which is defined as
$z \mathrm{I} (|z| \geq \sqrt{2 \eta})$.

Let us return the concrete Bernstein functions in Table~\ref{tab:exam}.
We are especially
interested in the KEP, LOG and LFR functions, because
there are analytic solutions for $\kappa(|z|)$ based on them.
Corresponding to  LOG and LFR,  $\kappa(|z|)$ are respectively
\begin{equation} \label{eqn:logan}
\kappa(|z|) =\frac{\alpha|z|-1 + \sqrt{(1+\alpha |z|)^2- {4 \lambda \alpha^2} }}{2 \alpha}
\end{equation}
  and
\begin{equation} \label{eqn:lfran}
\kappa(|z|) = \frac{2(\alpha |z| {+} 2)}{3 {\alpha}}  \cos\Big[\frac{1}{3}
\arccos\big( 1{-}{\lambda \alpha^2 } (\frac{3}{\alpha|z| {+} 2})^{3}  \big) \Big] {+} \frac{\alpha |z| {+} 2}{3 {\alpha}} {-}
\frac{2}{\alpha}.
\end{equation}
The derivation can be obtained by using direct algebraic computations.
We here omit the derivation details.
As for KEP, $\kappa(|z|)$ was derived by
\citet{ZhangTR:2013}.  That is,
\[
\kappa(|z|) = \frac{4(2\alpha |z| {+} 1)}{3}\cos^2\Big[\frac{1}{3 \alpha}
\arccos\big( {-} { \lambda \alpha^2}  (\frac{3}{2\alpha|z| {+} 1})^{\frac{3}{2}}  \big) \Big] {-} \frac{1}{\alpha}.
\]

\subsection{The Coordinate Descent Algorithm}
\label{sec:cda}

Based on the discussion in the previous subsection, the Bernstein penalty function is suitable
for the coordinate descent algorithm. We give the coordinate descent procedure in Algorithm~\ref{alg:coord}.
If the LOG and LFR functions are used, the corresponding  threshold  operators have the analytic forms in (\ref{eqn:logan})
and (\ref{eqn:lfran}).
Otherwise, we employ the bisection method for finding the root $\kappa(|z|)$.
The method is also very efficient.

When $\lambda \leq -\frac{1}{\alpha^2 \Phi{''}(0)}$ (or $\lambda > -\frac{1}{\alpha^2\Phi{''}(0)}$), we can obtain
that  $|\hat{b}|\leq |z|$ always holds. The
objective function $J_{1}(b)$ in (\ref{eqn:general}) is strictly convex in $b$ whenever $\lambda \leq -\frac{1}{\alpha^2 \Phi{''}(0)}$.
Moreover, according to Theorem~\ref{thm:nest}, the estimator
$\hat{b}$ in both the cases is strictly increasing w.r.t.\ $|z|$.
As we see, $P(b; \lambda)\triangleq \lambda \Phi(\alpha |b|)$ satisfies $P(b; \lambda)=P(-b; \lambda)$. Moreover, $P'(b; \lambda)$ is positive and uniformly
bounded  on $[0, \infty)$, and $\inf_{b} \;  P{''}(b; \lambda) >-1$ on $[0, \infty)$ when $\lambda < - \frac{1}{\alpha^2 \Phi{''}(0)}$. Thus,
the algorithm shares the same  convergence property as in \citet{MazumderSparsenet:11}.

\begin{algorithm}[!ht]
   \caption{The coordinate descent algorithm}
   \label{alg:coord}
\begin{algorithmic}
   \STATE {\bfseries Input:} $\{\x_i, y_i\}_{i=1}^n$ where each column of $\X=[\x_i, \ldots, \x_n]^T$ is standardized
   to have mean 0 and length 1,
   a grid of increasing values $\Lambda=\{\eta_1, \ldots, \eta_L\}$, a grid of decreasing values $\Gamma=\{\alpha_1, \ldots, \alpha_K\}$
   where $\alpha_K$ indexes the Lasso penalty. Set $\hat{\b}_{\alpha_K, \eta_{L+1}}=0$.
   \FOR{each value of $l \in \{L, L-1, \ldots, 1\}$ }
   \STATE Initialize  $\tilde{\b} = \hat{\b}_{\alpha_K, \eta_{l+1}}$;
    \FOR{each value of $k \in \{K, K-1, \ldots, 1\}$ }
    \IF{$\eta_l \leq - \frac{\Phi(\alpha_k)}{\alpha_k^2 \Phi{''}(0)}$ }
     \STATE Cycle through the following one-at-a-time updates
      \[\tilde{b}_j = S_{\alpha_{k}}\Big(\sum_{i=1}^n(y_i- z_{i}^j)x_{ij}, \eta_l\Big), \quad j=1, \ldots, p
       \]
      where $z_i^j=\sum_{k\neq j} x_{ik} \tilde{b}_k$, until the updates converge to $\b^{\ast}$;
   \STATE $\hat{\b}_{\alpha_k, \eta_l} \leftarrow \b^{\ast}$.
   \ENDIF
   \ENDFOR
   \STATE Increment $k$;
   \ENDFOR
   \STATE Decrement $l$;
   \STATE {\bfseries Output:} Return the two-dimensional solution $\hat{\b}_{\alpha, \eta}$ for $(\alpha, \eta) \in \Lambda{\times}\Gamma$.
\end{algorithmic}
\end{algorithm}

\section{Asymptotic Properties}
\label{sec:math}

We discuss asymptotic properties
of the sparse estimator. Following the setup of \citet{ZouLi:2008} and \citet{ArmaganDunsonLee},
we assume two conditions: (i) $y_i=\x_i^T \b^{*} + \epsilon_i$ where $\epsilon_1, \ldots, \epsilon_n$ are i.i.d.\ errors
with mean 0 and variance $\sigma^2$;
(ii) $\X^T \X/n\rightarrow \C$ where $\C$ is a positive definite matrix. Let ${\cal A}=\{j: b_{j}^{*} \neq 0\}$.
Without loss of generality, we assume that ${\mathcal A}=\{1, 2, \ldots, r\}$ with $r <p$. Thus, partition $\C$ as
\[
\C= \begin{bmatrix}\C_{11} & \C_{12} \\  \C_{21} & \C_{22} \end{bmatrix},
\]
where $\C_{11}$ is $r{\times} r$. Additionally,  let $\b^{*}_{1} =\{b^{*}_{j}: j \in {\mathcal A}\}$
and $\b^{*}_{2} = \{b^{*}_{j}:  j \notin {\cal A}\}$.

We are now interested in the  asymptotic behavior of the sparse estimator based on the penalty function $\Phi(\alpha |b|)$. That is,
\begin{equation} \label{eqn:sqre}
\tilde{\b}_n=\argmin_{\b}  \;  \|\y {-} \X \b\|_2^2 +  \lambda_n \sum_{j=1}^p    {\Phi(\alpha_n |b_j|)}.
\end{equation}
Furthermore, we let $\lambda_n=\frac{\eta_n}{\Phi(\alpha_n)}$ based on Theorem~\ref{thm:lp2}. For this estimator, we have the following  oracle property.

\begin{theorem}  \label{thm:oracle2} Let $\tilde{\b}_{n1}=\{\tilde{b}_{nj}: j \in \AM\}$ and   $\tilde{ {\cal A}}_n=\{j: \tilde{b}_{nj} \neq 0\}$. Suppose $\Phi(|b|)$ is a Bernstein function such that $\Phi(0)=0$ and $\Phi'(0)=1$, and there exists a constant $\gamma \in [0, 1)$
such that $\liml_{\alpha \to \infty} \frac{ \Phi'(\alpha)}{\alpha^{\gamma-1}} = c_0$ where $c_0 \in (0, \infty)$ when  $\gamma \in (0, 1)$ and $c_0 \in [0, \infty)$ when $\gamma=0$.
If $\eta_n/{n}^{\frac{\gamma_1}{2}} \rightarrow c_1 \in (0, \infty)$    and $\alpha_n/{n}^{\frac{\gamma_2}{2}} = c_2 \in (0, \infty)$ where $\gamma_1 \in (0, 1]$ for $\gamma =0$ or $\gamma_1 \in (0, 1)$ for $\gamma > 0$
and $\gamma_2 \in (0, 1]$  such that $\gamma_1{+}\gamma_2>1+\gamma \gamma_2$,
then $\tilde{\b}_n$ satisfies the following properties:
\begin{enumerate}
\item[\emph{(1)}]  Consistency in variable selection: $\liml_{n \rightarrow \infty} P( \tilde{{\cal A}}_n={\cal A})=1$.
\item[\emph{(2)}]  Asymptotic normality:  $\sqrt{n}(\tilde{\b}_{n1} - \b^{*}_{1}) \overset{d}{\longrightarrow} N(\0, \sigma^2 \C_{11}^{-1})$.
\end{enumerate}
\end{theorem}%

Obviously, the function $\Phi_{\rho}$ in (\ref{eqn:first}) satisfies the conditions
in the above theorem; that is, we see $\gamma = - \frac{\rho}{1-\rho}$ when $\rho\leq 0$ and $\gamma=0$ when $0< \rho\leq 1$ (see Proposition~\ref{pro:33}).  It follows from the condition $\liml_{\alpha \to \infty} \frac{ \Phi'(\alpha)}{\alpha^{\gamma-1}} = c_0$
that $\liml_{\alpha \to \infty} \frac{ \Phi(\alpha)}{\alpha^{\gamma}} = \frac{c_0}{ \gamma}$ for $\gamma\neq 0$. As a result, we obtain $\liml_{\alpha \to \infty} \frac{\alpha \Phi'(\alpha)}{\Phi(\alpha)} = {\gamma}$.
The condition  $\alpha_n/{n}^{\gamma_2/2} = c_2$ implies that $\alpha_n \to \infty$. Subsequently,
we have $\liml_{n \to \infty} \sum_{j=1}^p \frac{\Phi(\alpha_n |b_j|)} {\Phi(\alpha_n)} = \sum_{j=1}^p |b_j|^{\gamma}$ (see Theorem~\ref{thm:lp2}).
On the other hand, as stated earlier,  $\liml_{\alpha_n \to 0+} \sum_{j=1}^p \frac{\Phi(\alpha_n |b_j|)} {\Phi(\alpha_n)} = \liml_{\alpha_n \to 0+} \sum_{j=1}^p \frac{\Phi(\alpha_n |b_j|)} {\alpha_n} = \|\b\|_1$. Thus, we are also interested in the corresponding  asymptotic behavior of the sparse estimator.
In particular, we have the following theorem.

\begin{theorem}  \label{thm:asumptotic}
Let $\Phi(|b|)$ be a Bernstein function such that $\Phi(0)=0$ and $\Phi'(0)=1$. Assume
$\liml_{n \to \infty} \alpha_n =0$. If $\liml_{n \to \infty} \frac{\eta_n}{\sqrt{n}} = 2 c_3 \in [0, \infty)$,
then  $\tilde{\b}_{n} \overset{p}{\longrightarrow} \b^{*}$. Furthermore, if $\liml_{n \to \infty} \frac{\eta_n}{\sqrt{n}} =0$,
then $\sqrt{n}(\tilde{\b}_{n} {-} \b^{*})
\overset{d}{\longrightarrow} N(\0, \sigma^2 \C^{-1})$.
\end{theorem}%

In the previous discussion, $p$ is fixed. It  would be also interested in the asymptotic properties when $r$  and $p$ rely on $n$~\citep{ZhaoJMLR:06}. That is, $r \triangleq r_{n}$ and $p \triangleq p_n$ are allowed to grow as $n$ increases.
 Consider that $\tilde{\b}_n$
is the solution of the problem in (\ref{eqn:sqre}). Thus,
\[
0 \in  (\X \tilde{\b}_n {-} \y)^T \x_{\cdot j} + \frac{\eta_n \alpha_n \Phi'(\alpha_n |\tilde{b}_{nj}|)}{\Phi(\alpha_n)}  \partial |\tilde{b}_{nj}|, \quad j=1, \ldots, p.  \]
Under the condition $\alpha_n\to 0$, we have
\[
0 \in \lim_{n \to \infty} \Big\{ (\X \tilde{\b}_n {-} \y)^T \x_{\cdot j} + \frac{\eta_n \alpha_n \Phi'(\alpha_n |\tilde{b}_{nj}|)}{\Phi(\alpha_n)}  \partial |\tilde{b}_{nj}| \Big\}
= \lim_{n \to \infty} \Big\{ (\X \tilde{\b}_n {-} \y)^T \x_{\cdot j} + {\eta_n}  \partial |\tilde{b}_{nj}| \Big\}
\]
for $ j=1, \ldots, p$. Since the minimizer of the conventional lasso exists and unique (denote  $\hat{\b}_{0}$), the above
relationship implies that $\liml_{n \to \infty} \tilde{\b}_n=\liml_{n \to \infty} \hat{\b}_{0}$. Accordingly, we can obtain the same result as
in Theorem~4 of \citet{ZhaoYu:2006}.

Recently, \citet{ZhangZhang2012} presented a general theory of nonconvex regularization for sparse learning problems.
Their work is built on the following four conditions on the penalty function $P(b; \lambda)$: (i) $P(0; \lambda)=0$; (ii)  $P(-b; \lambda)=P(b; \lambda)$;
(iii) $P(b; \lambda)$ is increasing in $b$ on $[0, \infty)$; (iv) $P(b; \lambda)$ is subadditive w.r.t.\
$b\geq 0$, i.e., $P(s+t; \lambda) \leq P(s; \lambda)+ P(t; \lambda)$ for any $s \geq 0$ and $t \geq 0$.
It is easily seen  that the Bernstein function $\lambda \Phi(|b|)$ as a function of $b$ satisfies the first three conditions.
As for the fourth condition, it is also obtained via the fact that
\begin{align*}
\Phi(s+t) & = \int_{0}^{\infty}{ [1 - \exp(-(s+t) u)] \nu(d u) } \\
 & \leq \int_{0}^{\infty}{[1 - \exp(-s u) + 1 - \exp(-t u) ] \nu(d u) }
= \Phi(s) + \Phi(t), \quad \mbox{ for } s, t>0.
\end{align*}
Thus, we can directly apply the theoretical analysis of \citet{ZhangZhang2012}  to the Bernstein nonconvex penalty function.

\section{Bernstein Functions: A View of Concave Conjugate}
\label{sec:cc}

In this section we show that a Bernstein function can be
defined as a concave conjugate of some generalized distance function.
Given a function $f: S \subseteq \RB^p \rightarrow (-\infty, \infty)$, its concave conjugate, denoted
$g$, is defined by
\[
g(\v) = \inf_{\u \in S} \; \{\u^T \v - f(\u)\}.
\]
It is well known that $g$ is concave whether or not $f$ is concave. However, if $f$
is proper, closed and concave, the concave conjugate of $g$ is again $f$~\citep{BoydBook:2004}.
We apply this notion to explore Bernstein functions. Specifically, we show that Bernstein function
can be derived from a concave conjugate of some generalized distance function.

We are especially concerned with the generalized distance between two positive vectors.
One important family of such distances is the family of $\varphi$-divergences. We denote
$\RB_{+}^p=\{\u =(u_1, \ldots, u_p)^T \in \RB^p: u_j \geq 0 \mbox{ for } j=1, \ldots, p\}$
and $\RB_{++}^p=\{\u =(u_1, \ldots, u_p)^T \in \RB^p: u_j > 0 \mbox{ for } j=1, \ldots, p\}$.
Furthermore, if $\u \in \RB_{+}^p$ (or $\u \in \RB_{++}^p$),
we also denote $\u\geq 0$ (or $\u > 0$).
The definition  of the $\varphi$-divergence is now given as follows.

\begin{definition} \label{def:22} Let $\varphi: \RB_{++} \rightarrow \RB$ be twice continuously differentiable and strictly convex in
$\RB_{++}$ such that $\varphi(1) = \varphi'(1)=0$, $\varphi''(1)>0$ and $\lim_{a \rightarrow 0+} \varphi'(a) = - \infty$. For such a
function $\varphi$, the function $D_{\varphi}: \RB_{++}^p {\times} \RB_{++}^p \rightarrow \RB$ which is defined by
\[
D_{\varphi}(\u, \v)  \triangleq \sum_{j=1}^p v_j \varphi(u_j/v_j),
\]
is referred to as a $\varphi$-divergence.
\end{definition}

Note that when one only requires that convex function $\varphi(u)$ satisfies $\varphi(1)=0$, the resulting
distance function $D_{\varphi}$ is called a $f$-divergence~\citep{LieseVajda:1987,LieseIT:2006}. Thus, the $f$-divergence is a generalization
of the  $\varphi$-divergence.  The $f$-divergence  has  widely applied in statistical machine learning~\citep{NguyenANS:2009,ReidJMLR:2011}.
In the following theorem, we show that Bernstein functions can be defined as a concave conjugate
of $\varphi$-divergence.

\begin{theorem} \label{thm:cc} Assume that Bernstein function $\Phi(s)$ satisfies $\Phi(0)=0$ and $\Phi'(0)=1$.
Then there exists a  $\varphi$-divergence function $\varphi(v)$ from $\RB_{++}$ to $\RB$ such that
\[
\Phi(s) = \min_{w>0}  \left\{w s + \varphi(w) \right\}.
\]
\end{theorem}

\begin{corollary} \label{cor:cc} Assume that Bernstein function $\Phi(s)$ satisfies $\Phi(0)=0$ and $\Phi'(0)=1$.
Then there exists a  $\varphi$-divergence function $\varphi(v)$ from $\RB_{++}$ to $\RB$ such that
\[
\frac{\eta}{\alpha} \Phi(\alpha s) = \min_{w>0}  \left\{w s + \frac{\eta}{\alpha}  \varphi(w/\eta) \right\}.
\]
\end{corollary}

We now consider the Bernstein function $\Phi_{\rho}$ in (\ref{eqn:first}). Particularly,
it is induced by the following $\varphi$-function
\begin{equation}
\varphi_{\rho}(z)= \left\{ \begin{array}{ll} - \log z + z -1 & \mbox{ if } \rho=0, \\
\frac{z^{\rho}-\rho z + \rho -1}{\rho(\rho-1)} & \mbox{ if } \rho \neq 0 \mbox{ and } \rho \neq 1, \\
z \log z - z+1 & \mbox{ if } \rho =1,  \end{array} \right.
\end{equation}
where $\log 0=-\infty$ and $0\log 0=0$. This function was studied by~\citet{LieseVajda:1987,LieseIT:2006}. We can see that
$\varphi_{-1}(z)$ is the $\varphi$ function for KEP and $\varphi_{1/2}(z)$ is the $\varphi$ function for LFR.
Table~\ref{tab:exam1} shows that there is an interesting relationship between LOG and EXP; that is, both LOG and EXP
are respectively derived from the KL distance between $\eta$ and $w$ and the KL distance between $w$ and $\eta$.
This relationship has been established by \citet{ZhangNIPS:2012}.

It is worth pointing out that the concave conjugate of an arbitrary  $\varphi$-divergence is not always a Bernstein function.
For example, 
for any $\rho \in (-\infty, \infty)$,
$\varphi_{\rho}(z)$  still satisfies the conditions in Definition~\ref{def:22}. Let us take the case that $\rho>1$
and consider the corresponding concave conjugate; that is
\[
\tilde{g}(s) = \min_{w} \Big\{w s + \varphi_{\rho}(w)  \Big\}.
\]
It is direct to obtain for $\rho>1$
\[
\tilde{g}(s) = \left\{\begin{array}{ll} \frac{1}{\rho} & \mbox{ if } s\geq \frac{1}{\rho{-}1}, \\
\frac{1}{\rho}\Big[1- (1 {+} (1{-}\rho) s)^{\frac{\rho}{\rho{-}1} } \Big] & \mbox{ if } s < \frac{1}{ \rho{-}1},
\end{array} \right.
\]
which is not Bernstein. Specially, when $\rho=2$, we have
\[
M(s) = \left\{\begin{array}{ll} \frac{1}{2} & \mbox{ if } s\geq 1, \\
s - \frac{s^2}{2} & \mbox{ if } s < 1,  \end{array} \right.
\]
which is  the MCP function (see Eqn.(\ref{eqn:mcp})).
From Table~\ref{tab:exam1}, we see that both KEP and MCP are based on the $\chi^2$-distance~\citep{ZhangAAAI:2013,ZhangTR:2013}.


\begin{table}[!ht] 
\begin{center}
\caption{The corresponding $\varphi$-divergences $\varphi(z)$ and
generalized distances $D(\w, \etab)$ for the penalty functions $\Phi(s)$ in Table~\ref{tab:exam}. \label{tab:exam1}}
\begin{tabular}{l|l|l|l}
  \hline
 &   $\varphi(z)$ &   $D(\w, \etab)$ & \\ \hline
KEP  &  $\frac{1}{2}(z^{-1} {+} z {-} 2)$   & $\frac{1}{2}\sum_{j=1}^p \frac{(w_j-\eta_j)^2}{w_j}$ & $\chi^2$-distance \\
LOG  & $z - \log z - 1$ & $\sum_{j=1}^p \eta_j \log \frac{\eta_j}{w_j} {-}  \eta_j {+} w_j$ & Kullback-Leibler distance \\
LFR  & $ 2 (\sqrt{z} -1)^2$  &  $ 2\sum_{j=1}^p (\sqrt{w_j} - \sqrt{\eta_j})^2 $ & Hellinger distance \\
EXP  & $z \log z -z + 1$  & $\sum_{j=1}^p w_j \log \frac{w_j}{\eta_j} {-}  w_j {+} \eta_j$& Kullback-Leibler distance \\
MCP   & $\frac{1}{2}(z^{2} {-} 2z {+} 1)$ & $\frac{1}{2}\sum_{j=1}^p \frac{(w_j-\eta_j)^2}{\eta_j}$ & $\chi^2$-distance \\
\hline
\end{tabular}
\end{center}
\end{table}

\section{The CM Algorithm} \label{sec:cm}

The view of concave conjugate also leads us to a new approach for solving the penalized optimization problem.
Given a $\Phi(\alpha |b|)$, induced from a $\varphi$-divergence $D_{\varphi}$, as a penalty,
we consider the following regularization problem:
\begin{equation} \label{eqn:p3}
\min_{\b} \; \Big\{J(\b, \etab) \triangleq \frac{1}{2} \|\y - \X \b\|_2^2  + \frac{1}{\alpha} \sum_{j=1}^p \eta_j \Phi(\alpha |b_j|)  \Big\}.
\end{equation}
Clearly, when $\frac{\eta_1}{\alpha} = \frac{\eta_2}{\alpha} = \cdots = \frac{\eta_p}{\alpha} \triangleq \frac{\lambda}{\alpha}$,
the current penalized optimization problem becomes the conventional setting in Section~\ref{sec:lapexp}. In other words, the problem in (\ref{eqn:p3})
uses multiple tuning hyperparameters $\eta_j$ instead.
In terms of the discussion in the previous section,
we  equivalently reformulate (\ref{eqn:p3}) as
\begin{equation} \label{eqn:p2}
\min_{\b}  \min_{\w > 0} \Big\{ \frac{1}{2} \|\y - \X \b\|_2^2 + \w^{T} |\b| +  \frac{1}{\alpha} D_{\varphi}(\w, \etab) \Big\}.
\end{equation}

In this section, we deal with the problem (\ref{eqn:p2}) in which  $\etab$ is also a vector that needs to be estimated.
In particular,  we
develop a new estimation algorithm that we call \emph{conjugate-maximization}. 
We will see in our case that
the algorithm should be called \emph{conjugate-minimization}. Here we refer to as \emph{conjugate-maximization}  (CM)
in parallel with \emph{expectation-maximization \emph{(EM)}}.
The algorithm consists of two steps, which we
refer to as \emph{C-step} and \emph{M-step}.

We are given  initial values $\w^{(0)}$, e.g., $\w^{(0)}=\lambda (1, \ldots, 1)^T$ for some $\lambda>0$.
After the $k$th estimates $(\b^{(k)}, \etab^{(k)})$ of $(\b, \etab)$ are obtained,
the $(k{+}1)$th iteration of the CM algorithm is defined as follows.
\begin{description}
\item[\emph{C-step}] The C-step calculates $\w^{(k)}$ via
\[
\w^{(k)} = \argmin_{\w> 0} \; \Big\{C(\w|\b^{(k)}, \etab^{(k)}) \triangleq \sum_{j=1}^p w_j |b_j^{(k)}|  + \frac{1}{\alpha} D_{\varphi}(\w, \etab^{(k)})\Big\}.
\]
Since $D_{\varphi}(\w, \etab)$ is strictly convex in $\w$, this step is equivalent to finding the conjugate of $-D_{\varphi}/\alpha$
with respect to $|\b|$. We thus call it \emph{C-step}.

\item[\emph{M-step}] The M-step then calculates $\b^{(k{+}1)}$ and $\etab^{(k{+}1)}$  via
\[
(\b^{(k{+}1)}, \etab^{(k{+}1)} ) = \argmin_{\b, \; \etab} \; \Big\{ \frac{1}{2} \|\y - \X \b\|_2^2 + \sum_{j=1}^p w_j^{(k)} |b_j| + \frac{1}{\alpha} D_{\varphi}(\w^{(k)}, \etab) \Big\}.
\]
\end{description}
Note that given $\w^{(k)}$, $\b$ and $\etab$ are independent. Thus, the M-step can be partitioned into two parts.
Namely, $\etab^{(k{+}1)} = \argmin_{\etab} \; D_{\varphi}(\w^{(k)}, \etab)$ and
\[
\b^{(k{+}1)} = \argmin_{\b} \; \Big\{\frac{1}{2} \|\y - \X \b\|_2^2 + \sum_{j=1}^p w_j^{(k)} |b_j| \Big\}.
\]


We see that
the M-step in fact formulates a weighted $\ell_1$  minimization problem. It then can be immediately solved
by using existing methods such as the coordinate descent method and LARS.
Moreover, we directly have $\etab^{(k{+}1)}=\w^{(k)}$  in the M-step
due to that $D_{\varphi}(\w^{(k)}, \etab)=0$ if and only if $\etab^{(k{+}1)}=\w^{(k)}$.

We now give the C-steps.
Recall that
\[
P_{\alpha}(\b^{(k)}, \etab^{(k)}) \triangleq \sum_{j=1}^p \frac{\eta_j^{(k)}}{\alpha} \Phi(\alpha |b_j^{(k)}|)  = \min_{\w\geq 0} \; \{C(\w|\b^{(k)}, \etab^{(k)})\}.
\]
Since the minimizer of $\w$ is equal to the slope
of $P_{\alpha}(\b^{(k)}, \etab^{(k)})$ at the current $|\b^{(k)}|$, we can also calculate $\w^{(k)}$ via
\begin{equation} \label{eqn:c2}
\w^{(k)} = \nabla P_{\alpha}(\b^{(k)}, \etab^{(k)}).
\end{equation}
Hence, for the Bernstein function $\Phi_{\rho}(|b|)$ in (\ref{eqn:first}), we have
\[
w_{j}^{(k)} = {w_j^{(k{-}1)}} \big(1 + (1-\rho) \alpha |b_j^{(k)}| \big)^{-\frac{1}{1-\rho}}, \quad j=1, \ldots, p.
\]
Indeed,  the same method for the KEP, LOG and EXP penalty functions was developed by \citet{ZhangNIPS:2012} and \citet{ZhangTR:2013}.

%

\citet{ZouLi:2008} showed an equivalence of
LLA with the EM algorithm
under some conditions. In particular,
it is  the case for the log-penalty, which has an interpretation as a scale mixture of Laplace distributions
\citep{LeeCaronDoucetHolmes:2010,GarriguesfNIPS:2010}. In fact, the CM algorithm
bears an interesting resemblance to the EM algorithm, because we can treat $\w$ as missing data.
With such a treatment, the C-step of CM is related to the E-step of EM, which calculates
the expectations associated with missing data.  

There is a one-to-one correspondence between Bernstein  functions and Laplace exponents of subordinators
which are one-dimensional L\'{e}vy processes~\citep{SSVBook:2010}.
Recently, \citet{ZhangTR2:2013} developed a pseudo Bayesian approach for Bernstein nonconvex penalization. Moreover,
they gave an ECME (for expectation/conditional maximization either)~\citep{LiuBio:1994} for finding the sparse solution.

\subsection{Convergence Analysis} \label{sec:converg}

We now investigate the convergence of the CM algorithm. 
Noting that $\w^{(k)}$ is a function of $\b^{(k)}$ and $\etab^{(k)}$, we denote the objective function in the M-step by
\[
Q(\b, \etab|\b^{(k)}, \etab^{(k)}) \triangleq  \frac{1}{2} \|\y - \X \b\|_2^2 + \sum_{j=1}^p w_j^{(k)} |b_j| + \frac{1}{\alpha}D_{\varphi}(\w^{(k)}, \etab).
\]
We have the following lemma.

\begin{lemma} \label{lem:06} Let $\{(\b^{(k)}, \w^{(k)}): k=1, 2, \ldots \}$ be a sequence defined by the CM algorithm. Then,
\[
J(\b^{(k{+}1)}, \etab^{(k{+}1)}) \leq J(\b^{(k)}, \etab^{(k)}),
\]
with equality if and only if $\b^{(k{+}1)} = \b^{(k)}$ and $\etab^{(k{+}1)} = \etab^{(k)}$.
\end{lemma}
Since $J(\b^{(k)}, \etab^{(k)}) \geq 0$, this lemma shows that $J(\b^{(k)}, \etab^{(k)})$ converges monotonically to some $J^{*} \geq  0$.
In fact, the CM algorithm enjoys the same convergence
as the standard EM algorithm~\citep{Dempster:1977,WuEM:1983}.
Let ${\cal A}(\b^{(k)}, \etab^{(k)})$ be the set of values of $(\b, \etab)$ that minimize $Q(\b, \etab| \b^{(k)}, \etab^{(k)})$ over
$\Omega \subset \RB^p {\times} \RB_{++}^p$ and ${\cal S}$ be the set of stationary points of $J$
in the interior of $\Omega$. We can immediately follow from the Zangwill \emph{global convergence theorem} or
the literature~\citep{WuEM:1983,SriperumbudurNIPS:2009}
that

\begin{theorem} \label{thm:2} Let $\{\b^{(k)}, \etab^{(k)}\}$  be an sequence of the CM algorithm generated
by $(\b^{(k{+}1)}, \etab^{(k{+}1)}) \in {\cal A}(\b^{(k)}, \etab^{(k)})$. Suppose that
\emph{(i)} ${\cal A}(\b^{(k)}, \etab^{(k)})$ is closed over the complement
of ${\cal S}$ and that \emph{(ii)}
\[
J(\b^{(k{+}1)}, \etab^{(k{+}1)}) < J(\b^{(k)}, \etab^{(k)})
\quad \mbox{for all } (\b^{(k)}, \etab^{(k)}) \not\in {\cal S}.
\]
Then all the limit points of $\{\b^{(k)}, \etab^{(k)}\}$ are
stationary points of $J(\b, \etab)$ and $J(\b^{(k)}, \etab^{(k)})$ converges
monotonically to $J(\b^{*}, \etab^{*})$ for some stationary point
$(\b^{*}, \etab^{*})$. 
\end{theorem}

\section{Conclusion} \label{sec:conclusion}

In this paper we have exploited Bernstein functions in the definition of nonconvex penalty functions.
To the best of our knowledge, it is the first time that we apply theory of Bernstein functions
to systematically study nonconvex penalization problems.
We have shown that the Bernstein function has strong ability and attractive properties in sparse learning.
Geometrically, the Bernstein function holds  the property of regular variation. Theoretically,
it admits the oracle properties and can results in a unbiased and continuous sparse estimator. Computationally, the resulting
estimation problem can be efficiently solved by using the coordinate descent  and conjugate maximization algorithms.
We have illustrated  the KEP, LOG, EXP and LFR functions, which have wide applications in many scenarios but sparse modeling.



\appendix

\section{Several Important Results on Bernstein functions}

In this section we present several lemmas that are useful for Bernsterin functions.

\begin{lemma} \label{lem:lapexp}
Let $\Phi
(s)$ be a nonzero Bernstein function of $s$ on $(0, \infty)$. Assume $\liml_{s \to 0}\Phi(s)=0$ and $\liml_{s\to \infty} \frac{\Phi(s)}{s} =0$.
Then
\begin{enumerate}
\item[\emph{(a)}] $\liml_{s\rightarrow +\infty} \Phi^{(k)}(s) =0$ and  $\liml_{s\rightarrow 0+} s^k \Phi^{(k)}(s) =0$ for any $k\in \NB$. Additionally,
if $\liml_{s \to \infty} \Phi(s) < \infty$, then $\liml_{s\rightarrow \infty} s^k \Phi^{(k)}(s) =0$ for $k\in \NB$.
\item[\emph{(b)}]  If $\liml_{s \rightarrow \infty} s \Phi'(s)$ exists (possibly infinite), then $\liml_{s \rightarrow \infty} \frac{(-1)^{k{-}1}}{(k{-}1)!} s^k \Phi^{(k)}(s)$ for all $k \in \NB$
exist and are identical. Furthermore, if $\Phi'(0) = \liml_{s \to 0+} \Phi'(s) =1$, then  $\liml_{s \rightarrow \infty} s \Phi'(s) = \liml_{u\to 0+} \frac{F(u)}{u}$ where
$F(u)$ is the probability distribution  whose Laplace transform is $\Phi'(s)$.
\end{enumerate}
\end{lemma}

\begin{proof}
First,
it follows from the L\'{e}vy-Khintchine representation
that
\[
\Phi(s) =  \int_{0}^{\infty} \big[1- e^{- s u}  \big ] \nu(d u)
\]
due to $\Phi(0)=0$ and $\liml_{s \to \infty} \frac{\Phi(s)}{s}=0$.
Thus, we have
\[
\Phi^{(k)}(s) = (-1)^{k-1} \int_{0}^{\infty}  e^{- s u} u^{k} \nu(d u).
\]
When $s\geq k$ for any $k\in \NB$, it is easily verified  that $e^{-su} u^k \leq \frac{u^k}{1+u^k}$ for $u>0$. Note that
\[
\int_{0}^{\infty}{\min(u^k, 1) \nu(d u)} \leq \int_{0}^{\infty}{\min(u, 1) \nu(d u)} < \infty
\]
and
\[
\frac{u^k}{1+u^k} \leq \min(u^k, 1) \leq \frac{2 u^k}{1+u^k}, \quad u\geq 0.
\]
This implies that $\int_{0}^{\infty}{\min(u^k, 1) \nu(d u)} < \infty$ is equivalent to that  $\int_{0}^{\infty}{ \frac{u^k}{1+u^k} \nu(d u)} < \infty$.
As a result, we have that when $s\geq k$,
\[
 \int_{0}^{\infty}  e^{- s u} u^k \nu(d u) = \int_{0}^{\infty}  e^{- s u} u^k \nu(d u) \leq \int_{0}^{\infty}{ \frac{u^k}{1+u^k} \nu(d u)} < \infty.
\]
Thus,
\[
\lim_{s \to \infty} \Phi^{(k)}(s) = (-1)^{k-1} \lim_{s \to \infty} \int_{0}^{\infty}  e^{- s u} u^k \nu(d u)
= (-1)^{k-1} \int_{0}^{\infty} \lim_{s \to \infty}   e^{- s u} u^k  \nu(d u) =0.
\]

Additionally, since $e^{-s u} (su)^k \leq k^k e^{-k}$ for $s \geq 0$ and $u \geq 0$, we have
\begin{align}
\int_{0}^{\infty}  e^{- s u} (su)^k \nu(d u) &= \int_{0}^{1}  e^{- s u} (su)^k \nu(d u) +
 \int_{1}^{\infty}  e^{- s u} (su)^k \nu(d u) \nonumber \\
 & \leq\int_{0}^{1}  e^{- s u} (su)^k \nu(d u) + \int_{1}^{\infty}{ k^{k} e^{-k} \nu(d u)}. \label{eqn:a1}
\end{align}
Hence,
for any $s\leq 1$,
\[
\int_{0}^{\infty}  e^{- s u} (su)^k \nu(d u) \leq\int_{0}^{1}  u \nu(d u)
+ \int_{1}^{\infty}{ k^{k} e^{-k} \nu(d u)} \leq  \max(1,  k^{k} e^{-k}) \int_{0}^{\infty} {\min(1,u)  \nu(d u)} < \infty.
\]
As a result, we obtain
\[
\lim_{s \to 0} s^k \Phi^{(k)}(s) = (-1)^{k-1} \lim_{s \to 0} \int_{0}^{\infty}  e^{- s u} (su)^k \nu(d u)
= (-1)^{k-1} \int_{0}^{\infty} \lim_{s \to 0}  e^{- s u} (su)^k \nu(d u)=0.
\]

Furthermore,
$\liml_{s\to \infty}\Phi(s)=M_0<\infty$ implies that $\int_{0}^{\infty} {\nu(d u) } < \infty$, so we always have
\[
\int_{0}^{\infty}  e^{- s u} (su)^k \nu(d u) \leq  k^k e^{-k} \int_{0}^{\infty} {\nu(d u) } < \infty,
\]
which leads us to $\liml_{s\to \infty} s^k \Phi^{(k)}(s)=0$ for any $k\in \NB$.

We now prove Part (b). Consider  that
\[
\frac{(-1)^{k-1} s^k \Phi^{(k)}(s)}{(k{-}1)!}
=   \int_{0}^{\infty} \frac{s^{k}}{(k{-}1)!} e^{- s u} u^{k{-}1}  u \nu(d u)
\]
and that $ \frac{s^{k}}{(k{-}1)!} e^{- s u} u^{k{-}1}$ is the p.d.f.\ of gamma random variable $u$ with
shape parameter $k$ and scale parameter $1/s$. Such a gamma random variable converges to the Dirac Delta measure $\delta_0(u)$ in  distribution
as $s \to +\infty$.
For a fixed $u>0$,
$\frac{s^{k}}{(k{-}1)!} e^{- s u} u^{k{-}1}$ is monotone w.r.t.\ sufficiently large $s$.
Accordingly, using monotone convergence,
we have
 \begin{align*}
\lim_{s \to \infty} \frac{(-1)^{k-1} s^k \Phi^{(k)}(s)}{(k{-}1)!}
& =  0 \nu(\{0\})  + \lim_{s \to \infty}  \int_{0+}^{\infty} \frac{s^{k}}{(k{-}1)!} e^{- s u} u^{k{-}1}  u \nu(d u) \\
& = 0 \nu(\{0\})=  \int_{0}^{\infty} {\delta_0(u) u \nu(du)} = \liml_{s \rightarrow \infty} s \Phi'(s).
\end{align*}

When $\Phi'(0) = \liml_{s \to 0+} \Phi'(s) =1$, it is a well-known result that $\Phi'(s)$ is the Laplace transform of some probability distribution
(say, $F(u)$). That is,
\[
\Phi'(s) = \int_{0}^{\infty}{\exp(- s u) d F(u)}= \int_{0}^{\infty}{ s \exp(- s u) F(u) d u}.
\]

Recall that $s^2 u \exp(-su) \to $ $\delta_0(u)$ in distribution
as $s \to +\infty$. We  thus have
\[
\liml_{s \infty } s \Phi'(s) = \lim_{u \to 0+} \frac{F(u)}{u}.
\]
Furthermore, if $F(u)$ is the probability distribution of some continuous nonnegative random variable $U$, we have $\liml_{s \infty } s \Phi'(s) = F'(0+)$.
\end{proof}

\begin{lemma} \label{lem:88} Let $\Phi(s) \geq 0$ be a Bernstein function on $(0, \infty)$ such that $\Phi(0+)=0$ and $\Phi'(0+)=1$.
Then $\liml_{s \to +\infty} \frac{\Phi(s)}{\log (s)} =c < \infty$ if and only if there a sufficiently large positive number $M$ such that
$\frac{\Phi(s)}{\log (1+s)}$ is a decreasing function on $(M, \infty)$.
\end{lemma}
\begin{proof} Part ``$\Leftarrow$" is direct. Here we only prove ``$\Rightarrow$".  Owing to the properties of  $\Phi(s)$,
we have the L\'{e}vy representation of $\Phi(s)$ as follows
\[  \Phi(s) = \int_{0}^{\infty}{[1- \exp(-s u)] q(u) d u},
\]
where $q(u)$ is nonnegative and $\int_{0}^{\infty}{u q(u) d u} =1$ (because $\Phi'(s) = \int_{0}^{\infty}{\exp(-s u) u q(u) d u}$ and $1=\Phi'(0) = \int_{0}^{\infty}{u q(u) d u}$). Define
\[g(s) \triangleq  \frac{\Phi(s)}{\log (1+s)} = \int_{0}^{\infty}{ \frac{[1- \exp(-s u)]}{\log(1+s)}  q(u) d u}.
\]
Since $\liml_{s \to 0+} \frac{\Phi(s)}{\log (1+s)} =1$ and  $\mathop{\lim}\limits_{s \to +\infty} \frac{\Phi(s)}{\log (1+s)}=\mathop{\lim}\limits_{s \to +\infty} \frac{\Phi(s)}{\log (s)} < \infty$, we have that $\frac{\Phi(s)}{\log (1+s)}$ is bounded on $(0, \infty)$. Subsequently, we  can compute
\[
g'(s) = \frac{1}{(1{+}s)\log^2(1{+}s) }
 \int_{0}^{\infty}{\Big[\frac{u(1{+}s)\log(1{+}s)  +1}{\exp(s u) }- 1\Big]  q(u) d u}.
\]
Let $h(s) = \frac{u(1{+}s)\log(1{+}s)  +1}{\exp(s u) }- 1$ for $u\geq 0$.  Since $\mathop{\lim}\limits_{s \to \infty} h(s)=-1$ for $u>0$,
there exists a large $M_0$ such that $h(s) <0$ whenever $s>M_0$ and $u>0$. Additionally, $h(s) =0$ when $u=0$.
This implies that there exists a large $M$ such that $g'(s)\leq 0$ when  $s>M$; and
this completes the proof.
\end{proof}

\begin{lemma} \label{lem:01} Let $\Phi
(s)$ be a nonzero Bernstein function of $s$ on $(0, \infty)$ such that $\liml_{s \to \infty} s \Phi'(s)$ exists and it is finite. Then we have $\liml_{\s \to \infty} \frac{\Phi(s)}{\log(1+s)} = \liml_{s \to \infty} s \Phi'(s) < \infty$. Furthermore, we have
\[
\lim_{s \to \infty} \frac{s \Phi'(s)}{\Phi(s)} =0.
\]
\end{lemma}
\begin{proof}
It follows from the condition $\liml_{s \to \infty} s \Phi'(s) < \infty$ that
$\liml_{s \to \infty}\frac{\Phi(s)}{\log(1+s)}=
\liml_{s \to \infty}\frac{\Phi(s)}{\log(s)}=\liml_{s \to \infty} s \Phi'(s) < \infty$. Thus, when $\liml_{s \to \infty} \Phi(s) = \infty$, we have
$ \lim_{s \to \infty} \frac{s \Phi'(s)}{\Phi(s)} =0$.
Otherwise $\liml_{s \to \infty} \Phi(s) =M \in (0, \infty)$,  we always have that
$\liml_{a \to \infty}\frac{\Phi(s)}{\log(1+s)}=\liml_{s \to \infty}\frac{\Phi(s)}{\log(s)}= \liml_{s \to \infty} s \Phi'(\alpha) = 0$.
Thus, we  have $\liml_{s \to \infty} \frac{s \Phi'(s)}{\Phi(s)}=0$ in any cases.
\end{proof}

\begin{lemma} \label{lem:03} Let $\Phi
(s)$ be a nonzero Bernstein function of $s$ on $(0, \infty)$. Assume  $\Phi(0)=0$, $\Phi'(0)=1$, and $\Phi'(\infty)=0$.  If $\liml_{s \to \infty} \frac{s \Phi'(s)}{\Phi(s)}$ exists, then $\liml_{s \to \infty} \frac{s \Phi'(s)}{\Phi(s)} \in [0, 1)$.
\end{lemma}
\begin{proof} Consider that $s\Phi'(s)-\Phi(s)$ is a decreasing function on $(0, \infty)$
because its first-order derivative is non-positive; i.e.,  $s \Phi{''}(s)\leq 0$. As a result,
we have $0\leq\frac{s \Phi'(s)}{\Phi(s)} \leq 1$. Subsequently, $\gamma=\liml_{s \to \infty}
\frac{s \Phi'(s)}{\Phi(s)} \in [0, 1]$.

We are now to prove that $\gamma$ should be smaller than $1$. Note that when $\liml_{s \to \infty} \Phi(s)< \infty$, we have that $\liml_{s \to \infty} s \Phi'(s) =0$ (see Lemma~\ref{lem:lapexp}). Hence, $\liml_{s\to \infty} \frac{s \Phi'(s)}{\Phi(s)}=0 <1$. Thus, we now consider the case that  $\liml_{s \to \infty} \Phi(s)=\infty$.
We define $h(s) \triangleq \log(1+\Phi(s))$, which is also Bernstein because the composition of two Bernstein functions are still Bernstein.
Moreover, we have $h(0)=1$,  $h'(0)=1$ and $h'(\infty)=0$. Additionally,
\[
\lim_{s\to \infty} \frac{\log(1+\Phi(s))}{\log(1+s)} = \lim_{s\to \infty} \frac{(1+s) \Phi'(s)}{1+\Phi(s)} = \lim_{s\to \infty} \frac{s \Phi'(s)}{\Phi(s)} \leq 1.
\]
It  then follows from Lemma~\ref{lem:88} that there a sufficiently large positive number $M_1$ such that
$\frac{\log(1+\Phi(s))}{\log (1+s)}$ is a decreasing function on $[M_1, \infty)$.
Recall that
\[
\lim_{s\to \infty} \frac{\Phi(s)}{s} = \lim_{s\to \infty} \Phi'(s) =0,
\]
which implies that there a sufficiently large positive number $M_2$ such that $\Phi(s) \ll s$ for $s \geq M_2$.
Let $M=\max(M_1, M_2)$. We have $1+ \Phi(M) < 1 + M$ and $\frac{\log(1+\Phi(M))}{\log(1+M)}<1$. Moreover, for any $s>M$,
\[
\frac{\log(1+\Phi(s))}{\log(1+s)} \leq \frac{\log(1+\Phi(M))}{\log(1+M)}<1.
\]
Accordingly, we obtain
\[
\lim_{s\to \infty} \frac{s \Phi'(s)}{\Phi(s)} = \lim_{s\to \infty} \frac{\log(1+\Phi(s))}{\log(1+s)} \leq \frac{\log(1+\Phi(M))}{\log(1+M)}<1.
\]
%
%
\end{proof}

\section{The Proof of Theorem~\ref{thm:lp2}}

\begin{proof}
It is directly verified that
\[
\lim_{\alpha \to 0} \frac{\Phi(\alpha |b|)}{\Phi(\alpha)} = \liml_{\alpha \to 0} \frac{|b| \Phi'(\alpha |b|)}{\Phi'(\alpha)}
=\frac{|b| \Phi'(0)}{\Phi'(0)} =|b|
\]
due  to $\Phi'(0) = 1 \in (0, \infty)$. Clearly, we have that $\liml_{\alpha \to +\infty} \frac{\Phi(\alpha s)}{\Phi(\alpha)} = 0$ when $s=0$ and
that $\liml_{\alpha \to +\infty} \frac{\Phi(\alpha s)}{\Phi(\alpha)} = 1$ when $s=1$.

Lemma~\ref{lem:03} shows that $\gamma=\liml_{s \to \infty}
\frac{s \Phi'(s)}{\Phi(s)} \in [0, 1)$. When $\liml_{s \to \infty} \frac{\Phi(s)}{\log(1+s)} < \infty$, Lemma~\ref{lem:01} implies that $\gamma=0$.
According to Theorem~1 in Chapter VIII.9 of \cite{FellerBook:1971},
we have the second part of the theorem.
\end{proof}

\section{The Proof of Proposition~\ref{pro:33}}

\begin{proof} Let $\omega =\frac{1}{1-\rho}$. For $-\infty<\rho\leq 1$, we have $\omega \ (0, \infty]$. We now write $\Phi'_{\rho}(s)$ for a fixed $s>0$ as  $1/g(\omega)$ where
\[
g(\omega) = {(1 + \frac{s}{\omega})^{\omega}}.
\]
It is a well-known result that for a fixed $s>0$ $g(\omega)$ is increasing in $\omega$ on $(0, \infty)$. Moreover, $\liml_{\omega \to \infty} g(\omega) = \exp(s)$.  Accordingly, $\Phi'_{\rho}(s)$  is decreasing in $\rho$ on $(-\infty, 1]$. Moreover, we obtain
\[
\Phi_{\rho_1} (s) = \int_{0}^{s}{\Phi'_{\rho_1}(t) d t } \geq \int_{0}^{s}{\Phi'_{\rho_2}(t) d t } = \Phi_{\rho_2} (s)
\]
whenever $\rho_1\leq \rho_2 \leq 1$.

The proof of Part-(b) is immediately. We here omit the details.
\end{proof}

\section{The Proof of Theorem~\ref{thm:sparsty}}

\begin{proof}
The first-order derivative of (\ref{eqn:general}) w.r.t.\ $b$ is
\[
\sgn(b)\big(|b| + {\lambda} \Phi'(|b|) \big) - z.
\]
Let $g(|b|)= |b| + {\lambda} \Phi'(|b|)$. It is clear that if
$|z|< \min_{b\neq 0}\{g(|b| \}$, the resulting estimator is 0; namely, $\hat{b}=0$.
We now check the minimum value
of $g(s)=s + {\lambda} \Phi'(s)$ for $s\geq 0$.

Taking the first-order derivative of $g(s)$ w.r.t.\ $s$, we have
\[
g'(s) = 1 + {\lambda} \Phi{''}(s).
\]
Note that $\Phi{''}(s)$ is non-positive and increasing in $s$. As a result, we have
\[
g'(s) \geq 1 + {\lambda} \Phi{''}(0).
\]
Thus, if $\lambda \leq -\frac{1}{\Phi{''}(0)}$, $g(s)$  attains its minimum value ${\lambda}\Phi'(0)$ at $s^{*}=0$.
Otherwise,  $g(s)$ attains its minimum value when $s^{*}$ is the solution of
 $1 + {\lambda} \Phi{''}(s)=0$.

First, we consider the case that $\lambda \leq -\frac{1}{\Phi{''}(0)}$. In this case, the resulting estimator is 0 when $|z|\leq {\lambda} \Phi{'}(0)$. If $z> {\lambda}\Phi{'}(0)$, then the resulting estimator should be a positive root of the equation
$b + {\lambda} \Phi{'}(b) - z = 0$ in $b$. Letting $h(b)=b + {\lambda} \Phi{'}(b) - z$,
we study the roots of $h(b)=0$.
Note that
$h(z) =  {\lambda} \Phi{'}(z) >0$ and
$h(0)=  {\lambda} \Phi{'}(0) - z <0$. In this case, moreover, we have that $h(b)$ is increasing on $[0, \infty)$.
This implies that  $h(b)=0$
has one and only one positive root.
Furthermore, the resulting estimator $0<\hat{b}<z$ when $z> {\lambda}\Phi{'}(0)$.
Similarly, we can obtain that $z<\hat{b}<0$ when $z<- {\lambda} \Phi{'}(0)$.
As stated in \cite{Fan01}, a sufficient and necessary condition for ``continuity" is the
the minimum of
$|b|+ {\lambda} \Phi'(|b|)$ is attained at $0$. This implies that that the resulting estimator is continuous.

Next, we  prove the case that $\lambda > -\frac{1}{\Phi{''}(0)}$. In this case, $g(s)$ attains its  minimum value
$g(s^*)= s^* + {\lambda} \Phi'(s^*)$ when $s^*$ is the solution of equation
$1 + {\lambda} \Phi{''}(s)=0$. Note that $\Phi{''}(s)$ is non-positive and increasing in $s$.
Thus, the solution $s^*$ exists and
is unique. Moreover, since $\Phi{''}(s^*) = - \frac{1}{\lambda} > \Phi{''}(0)$, we have $s^*> 0$.
In this case, the resulting estimator is 0 when $|z|\leq s^* + {\lambda} \Phi'(s^*)$.
We just make attention on the case that $|z|> s^* + {\lambda} \Phi'(s^*)$.
Subsequently, the resulting estimator is $\hat{b}=\sgn(z)\kappa(|z|)$
where $\kappa(|z|)$ should be a positive root of equation $b + {\lambda} \Phi{'}(b) - |z| = 0$.
We now need to prove that
$\kappa(|z|)$  exists and is unique on $(s^*, |z|)$.
We have that $h(b)=b + {\lambda} \Phi{'}(b) - |z| $  is a convex
function of $b$ on $[0, \infty)$ due to  $h{''}(b)= {\lambda} \Phi{'''}(b) \geq 0$.
This implies that $h(b)$ is increasing on $[s^*, \infty)$ and decreasing on $(0, s^*)$.
Thus, the equation $h(b)=0$ has at most two positive roots, which are on $(0, s^*)$ or $[s^*, \infty)$.
Since $h(s^*)= s^* + {\lambda} \Phi'(s^*)-|z|<0$
and $h(|z|)= {\lambda} \Phi'(|z|) \geq 0$, the equation $h(b)=0$ has an unique root on $(s^*, |z|)$.
Thus, $\kappa(|z|)$ exists and is unique on $(s^*, |z|)$. It is worth pointing
out that if the equation $h(b)=0$ has a root on $(0, s^*)$, the objective function $J_1(b)$ attains its maximum value at this root.
Thus, we can exclude this root.
\end{proof}

\section{The Proof of Proposition~\ref{thm:nest}}

Observe that $1 = \Phi'(0)=\int_{0}^{\infty}{u \nu(d u)}$ and $\Phi(\alpha)= \int_{0}^{\infty}{(1-\exp(-\alpha u)) \nu(d u)}$.
Since $\alpha u>1-\exp(-\alpha u)$ for $u>0$, we obtain $\Phi(\alpha)<\alpha$.
Additionally, $\Big[\frac{\alpha}{\Phi(\alpha)} \Big]' =\frac{\Phi(\alpha) - \alpha \Phi'(\alpha)}{\Phi^2(\alpha)} \geq 0$
due to $[\Phi(\alpha) - \alpha \Phi'(\alpha)]' = -\Phi{''}(\alpha) \geq 0$. Also, $\Big[\frac{1}{\Phi(\alpha)} \Big]' \leq 0$.
We thus obtain that $\frac{\alpha}{\Phi(\alpha)}$ is increasing, while $\frac{1}{\Phi(\alpha)}$ is decreasing. Furthermore,
we can see that $\liml_{\alpha \to 0+} \frac{\alpha }{\Phi(\alpha)}=\liml_{\alpha \to 0+} \frac{1 }{\Phi'(\alpha)}=1$ and $\liml_{\alpha \to \infty} \frac{\alpha }{\Phi(\alpha)} = \liml_{\alpha \to \infty} \frac{1}{\Phi'(\alpha)}= \infty$.

\section{The Proof of Theorem~\ref{thm:0-1}}

\begin{proof}
First, it is easily obtained that
$\lim\limits_{\alpha \to 0} \frac{\alpha} {\Phi(\alpha)} = \frac{1}{\Phi'(0)}$
and $\lim\limits_{\alpha \to 0}  \frac{\Phi(\alpha)}{\alpha^2} = \infty$. This implies that in the limiting case the condition $\eta \leq -\frac{\Phi(\alpha)}{\alpha^2 \Phi{''}(0)}$ is always met (i.e., Case (i) in Theorem~\ref{thm:sparsty}).
Moreover, $|z| > \frac{\eta \alpha}{\Phi(\alpha)} \Phi{'}(0)$ degenerates to $|z|> \eta$.
In addition, we have
\[
\lim_{\alpha \to 0} \frac{\alpha \Phi'(\alpha b)}{\Phi(\alpha)} = \lim_{\alpha \to 0} \frac{\Phi'(\alpha b) + \alpha b \Phi{''}(\alpha b)}{\Phi'(\alpha)} =1.
\]
This implies that $\kappa(|z|)$ converges to the nonnegative solution of equation of the form
\[
b + \eta - |z|=0.
\]
That is,  $\kappa(|z|)=|z| - \eta$ when $|z|>\eta$.

Second, it is easily obtained that
$\liml_{\alpha \to \infty} \frac{\alpha} {\Phi(\alpha)} = \infty$
and $\liml_{\alpha \to \infty}  \frac{\Phi(\alpha)}{\alpha^2} = 0$.
This implies that in the limiting case the condition $\eta > -\frac{\Phi(\alpha)}{\alpha^2 \Phi{''}(0)}$ is always held.

Recall that $s^*>0$ is the unique root of $1+ {\lambda} \Phi{''}(s)=0$ and $\Phi{''}(s)$ is monotone increasing, so we can express
$s^*$ as $s^*= \frac{1}{\alpha} (\Phi{''})^{-1}(- \Phi(\alpha)/(\eta \alpha^2))$. Since $\liml_{\alpha \to \infty} \Phi(\alpha)/(\eta \alpha^2) =0$,
we can deduce that  $\liml_{\alpha \to \infty} (\Phi{''})^{-1}(- \Phi(\alpha)/(\eta \alpha^2))=\infty$. Subsequently,
\[
\lim_{\alpha \to \infty} s^{*} =\lim_{\alpha \to \infty}  \frac{1}{\alpha} (\Phi{''})^{-1}(- \Phi(\alpha)/(\eta \alpha^2))= \lim_{\alpha \to \infty}
\big[(\Phi{''})^{-1}(- \Phi(\alpha)/(\eta \alpha^2))\big]' \leq |z|.
\]
Additionally,
\begin{align*}
\lim_{\alpha \to \infty} \frac{\eta \alpha}{\Phi(\alpha)} \Phi'[(\Phi{''})^{-1}(- \Phi(\alpha)/(\eta \alpha^2))]
&=  \lim_{\alpha \to \infty}
- \frac{[\Phi(\alpha)/\alpha^2]}{\frac{\Phi'(\alpha)}{\alpha} - \frac{\Phi(\alpha)}{\alpha^2} }  \big[(\Phi{''})^{-1}(- \Phi(\alpha)/(\eta \alpha^2))\big]' \\
&= \lim_{\alpha \to \infty}
\big[(\Phi{''})^{-1}(- \Phi(\alpha)/(\eta \alpha^2))\big]' = \lim_{\alpha \to \infty} s^{*}.
\end{align*}
Assume $ \liml_{\alpha \to \infty} s^{*} =c \in (0, |z|]$. Then for sufficiently large $\alpha$, we have
$(\Phi{''})^{-1}(- \Phi(\alpha)/(\eta \alpha^2)) \simeq \alpha$; that is,
\[
\Phi(\alpha) \simeq  - \eta \alpha^2 \Phi{''}(\alpha).
\]
However, if $\liml_{\alpha \to \infty} \Phi(\alpha) <\infty$ then
$-\liml_{\alpha \to \infty} \alpha^2 \Phi{''}(\alpha) =  \liml_{\alpha \to \infty}  \frac{\Phi(\alpha)}{\log(\alpha)}=0$; while $\liml_{\alpha \to \infty} \Phi(\alpha) = \infty$ then
$-\liml_{\alpha \to \infty} \alpha^2 \Phi{''}(\alpha) =\liml_{\alpha \to \infty} \alpha \Phi{'}(\alpha) = \liml_{\alpha \to \infty}  \frac{\Phi(\alpha)}{\log(\alpha)} < \infty$.
This makes the contradiction due to the assumption
$\liml_{\alpha \to \infty} s^{*} =c \in (0, |z|]$. Thus, we have  $\liml_{\alpha \to \infty} s^{*} =0$. Hence,
\[
\lim_{\alpha \to \infty} s^{*} + \frac{\eta \alpha}{\Phi(\alpha)} \Phi'(\alpha s^*) =0.
\]
Finally, we have
\[
\lim_{\alpha \to \infty} \kappa(b) - \frac{\eta \alpha}{\Phi(\alpha)} \Phi'(\alpha \kappa(b)) = |z|,
\]
which implies $\liml_{\alpha \to \infty} \kappa(|z|) = |z|$. The second part now follows.
\end{proof}

\section{The Proof of Theorems~\ref{thm:oracle2} and \ref{thm:asumptotic}}


The proof is similar to that of Theorem~1 in \cite{ArmaganDunsonLee}.
Let $\tilde{\b}_n =\b^{*} + \frac{\hat{\u} }{\sqrt{n}}$ and
\[
\hat{\u} = \argmin_{\u} \; \bigg\{G_n(\u) \triangleq  \Big\|\y - \X (\b^{*} +\frac{\u}{\sqrt{n}} ) \Big\|_2^2
    + \eta_n \sum_{j=1}^p \frac{\Phi(\alpha_n|b^{*}_j {+} \frac{u_j}{\sqrt{n}}|) }  {\Phi(\alpha_n)} \bigg\}.
\]
Then $\hat{\u}  = \sqrt{n}(\tilde{\b}_n - \b^{*})$.
Consider that
\[
 G_n(\u) - G_n(\0)
 = \u^T(\X^T \X/n) \u - 2 \u^T  \frac{\X^T \epsi}{\sqrt{n}}+
  \eta_n \sum_{j=1}^p
\frac{\Phi(\alpha_n|b^{*}_j {+} \frac{u_j}{\sqrt{n}}|) {-} \Phi(\alpha_n|b^{*}_j|) } {\Phi(\alpha_n)}.
\]
Clearly, $\X^T \X/n \rightarrow \C$ and $\frac{\X^T\epsi}{\sqrt{n}} \overset{d}{\rightarrow} \z  \overset{d}{=} N(\0, \sigma^2 \C)$.
We now discuss the limiting behavior of the third term of the right-hand side. 

We partition $\z$ into $\z^T=(\z_1^T, \z_2^T)$ where  $\z_1=\{z_j: j \in {\cal A}\}$ and $\z_2=\{z_j: j \notin {\cal A}\}$.
First, assume $b^{*}_j=0$. The previous results imply
\[
 \eta_n \frac{\Phi(|u_j| \frac{\alpha_n}{\sqrt{n}})} {\Phi(\alpha_n)}
\backsimeq  \frac{n^{\frac{\gamma_1 {+} \gamma_2 {-} 1}{2}}}{n^{\frac{\gamma_2 \rho}{2} }} \frac{\eta_n}{n^{\frac{\gamma_1}{2} }} \frac{\alpha_n}{n^{\frac{\gamma_2}{2} }}  \frac{n^{\frac{\gamma_2 \rho}{2} }}{\alpha_n^{\rho}}  \frac{\alpha_n^{\rho} } {\log (\alpha_n)}  \frac{ \log (\alpha_n) } {\Phi(\alpha_n)}  \frac{\Phi \big(|u_j| \frac{ \alpha_n}{\sqrt{n} } \big) }
 {\frac{ \alpha_n}{\sqrt{n} }} \rightarrow +\infty
\]
whenever $\gamma=0$, due to $\liml_{\alpha \to \infty} \frac{\log(\alpha)}{\Phi(\alpha)} = \liml_{\alpha \to \infty} \frac{1}{\alpha \Phi'(\alpha)} = \frac{1}{c_0}>0$. Here we take $\rho$ as a positive constant such that $\rho\leq \frac{\gamma_1 {+} \gamma_2 {-} 1}{\gamma_2}$. If $\gamma \in (0, 1)$, we also have
\[
 \eta_n \frac{\Phi(|u_j| \frac{\alpha_n}{\sqrt{n}})} {\Phi(\alpha_n)}
\backsimeq \frac{n^{\frac{\gamma_1 +\gamma_2- 1}{2}}}{n^{ \frac{\gamma_2 \gamma}{2}}}  \frac{ \alpha_n^{\gamma} } {\Phi(\alpha_n)}  \frac{\Phi\big(|u_j| \frac{\alpha_n}{\sqrt{n}} \big) }
 { \frac{\alpha_n}{\sqrt{n}}} \rightarrow +\infty,
\]
because $\liml_{\alpha \to \infty} \frac{\alpha^{\gamma}}{\Phi(\alpha)} = \liml_{\alpha \to \infty} \frac{\gamma \alpha^{\gamma-1}}{ \Phi'(\alpha)} = \frac{\gamma}{c_0}>0$.

Next, we assume that $b^{*}_j\neq 0$. Subsequently, for sufficiently large $n$,
\begin{align} \label{eqn:99}
& \eta_n \frac{\Phi(\alpha_n|b^{*}_j {+} \frac{u_j}{\sqrt{n}}|) - \Phi(\alpha_n|b^{*}_j|) } {\Phi(\alpha_n)} \nonumber  \\
 & =   \eta_n \frac{\Phi(\alpha_n (b^{*}_j {+} \frac{u_j}{\sqrt{n}})
 \sgn(b^{*}_j) ) {-} \Phi(\alpha_n b^{*}_j \sgn(b^{*}_j)) } {\Phi(\alpha_n)}  \nonumber \\
& =  \frac{u_j}{b^{*}_j {+} \theta \frac{u_j}{\sqrt{n}}} \frac{\eta_n }{\sqrt{n}}  \frac{ \Phi'\Big(\alpha_n (b^{*}_j {+} \theta \frac{u_j}{\sqrt{n}}) \sgn(b^{*}_j) \Big)  \alpha_n (b^{*}_j {+} \theta \frac{u_j}{\sqrt{n}}) \sgn(b^{*}_j)} {\Phi(\alpha_n)}  \quad \{\mbox{for some }  \theta \in (0, 1) \}  \\
& \to 0. \nonumber
\end{align}
Here we use the fact that $\liml_{z \to \infty} \frac{z \Phi'(z)}{\Phi(z)} = \gamma \in [0, 1)$.

By Slutsky's theorem, we have
\[
 G_n(\u) - G_n(\0)  \overset{d}{\rightarrow} \left\{\begin{array}{ll} \u_1^T \C_{11} \u_1 - 2 \u_1^T  \z_1 & \mbox{ if } u_j=0 \; \forall j \notin {\cal A},
 \\ \infty & \mbox{ otherwise}. \end{array} \right.
\]
This implies that $G_n(\u) - G_n(\0)$ converges in distribution to a convex function, whose unique minimum is
$(\C_{11}^{-1} \z_1, \0)^T$. It then follows from  epiconvergence \citep{KnightFu:2000} that
\begin{equation} \label{eqn:11}
\hat{\u}_1 \overset{d}{\rightarrow} \C_{11}^{-1} \z_1 \; \mbox{ and } \;  \hat{\u}_2 \overset{d}{\rightarrow} \0.
\end{equation}
This proves asymptotic normality due to $\z_1 \overset{d}{=} N(\0, \sigma^2 \C_{11})$.

Recall that $\tilde{b}_{n j} \overset{p}{\rightarrow} b^{*}_j$ for any $j \in \AM$, which implies that $\Pr(j \in \AM_n) \to 1$.
Thus, for consistency in Part (1), it suffices to obtain  $\Pr(l \in \AM_n) \to 0$ for any $l \notin \AM$.
For such an event ``$l\in \AM_n$," it follows from the KKT optimality conditions
that $2 \x_{l}^T(\y - \X \tilde{\b}_{n})=\frac{\eta_n \alpha_n \Phi'(\alpha_n |\tilde{b}_{nj}|)}{\Phi(\alpha_n)}$.
Note that
\[
\frac{2 \x_l^T (\y -\X \tilde{\b}_{n})}{\sqrt{n}} = 2 \frac{\x_l^T \X \sqrt{n}(\b^{*}-\tilde{\b}_n)}{n} + \frac{2 \x_l^T \epsi}{\sqrt{n}},
\]
and $\liml_{n\to \infty} \frac{\eta_n \alpha_n \Phi'(\alpha_n |\tilde{b}_{nj}|)}{\sqrt{n} \Phi(\alpha_n)}
= \liml_{n\to \infty} \frac{\eta_n \alpha_n \Phi'( \sqrt{n} |\tilde{b}_{nj}| \alpha_n/\sqrt{n})}{\sqrt{n} \Phi(\alpha_n)}
\backsimeq \liml_{n\to \infty} \frac{n^{\gamma_1 {+} \gamma_2 {-} \frac{1}{2}} }{\gamma_2 \log(n)} \frac{\log(\alpha_n)} {\Phi(\alpha_n)} \to \infty$ for $\gamma=0$ or $\liml_{n\to \infty} \frac{\eta_n \alpha_n \Phi'(\alpha_n |\tilde{b}_{nj}|)}{\sqrt{n} \Phi(\alpha_n)}
= \liml_{n\to \infty} \frac{\eta_n \alpha_n \Phi'( \sqrt{n} |\tilde{b}_{nj}| \alpha_n/\sqrt{n})}{\sqrt{n} \Phi(\alpha_n)}
\backsimeq \liml_{n\to \infty} \frac{n^{\gamma_1 {+} \gamma_2 {-} \frac{1}{2}} }{n^{\frac{\gamma \gamma_2}{2} } } \frac{ \alpha_n^{\gamma} } {\Phi(\alpha_n)} \to \infty$ for $\gamma>0$
due to $\sqrt{n} |\tilde{b}_{nj}| \overset{p}{\rightarrow}  0$ by (\ref{eqn:11}) and Slutsky's theorem. Accordingly, we have
\[
\Pr(l \in \AM_n) \leq \Pr\Big[ 2 \x_{l}^T(\y - \X \tilde{\b}_{n})=\frac{\eta_n \alpha_n \Phi'(\alpha_n |\tilde{b}_{nj}|)}{\Phi(\alpha_n)}\Big]
\to 0.
\]

As for the proof of Theorem~\ref{thm:asumptotic}, we consider the case that $\liml_{n\to \infty} \alpha_n =0$. In this case, we have
\[
\lim_{n \to \infty} \frac{\Phi(\alpha_n/\sqrt{n})}{\alpha_n/\sqrt{n}} = 1 \; \mbox{ and } \;
\lim_{n \to \infty} \frac{\Phi(\alpha_n)}{\alpha_n} = 1.
\]
Assume that $\liml_{n\to \infty} \eta_n/\sqrt{n}= 2 c_3 \in [0, \infty]$. Then
\[
 \eta_n \frac{\Phi(|u_j| \frac{\alpha_n}{\sqrt{n}})} {\Phi(\alpha_n)}= |u_j|  \frac{\eta_n}{\sqrt{n}}
 \frac{\alpha_n  \Phi(|u_j| \frac{\alpha_n}{\sqrt{n}})} {\Phi(\alpha_n) |u_j| {\alpha_n}/{\sqrt{n}}} \to 2 c_3 |u_j|
\]
when $u_j\neq 0$. If $b^{*}_j\neq 0$, then
\begin{align*}
& \eta_n \frac{\Phi(\alpha_n|b^{*}_j {+} \frac{u_j}{\sqrt{n}}|) - \Phi(\alpha_n|b^{*}_j|) } {\Phi(\alpha_n)} \nonumber  \\
 & =   \eta_n \frac{\Phi(\alpha_n (b^{*}_j {+} \frac{u_j}{\sqrt{n}})
 \sgn(b^{*}_j) ) {-} \Phi(\alpha_n b^{*}_j \sgn(b^{*}_j)) } {\Phi(\alpha_n)}  \nonumber \\
& =  \frac{u_j}{b^{*}_j {+} \theta \frac{u_j}{\sqrt{n}}} \frac{\eta_n }{\sqrt{n}}  \frac{ \Phi'\Big(\alpha_n (b^{*}_j {+} \theta \frac{u_j}{\sqrt{n}}) \sgn(b^{*}_j) \Big)  \alpha_n (b^{*}_j {+} \theta \frac{u_j}{\sqrt{n}}) \sgn(b^{*}_j)} {\Phi(\alpha_n)}  \quad \{\mbox{for some }  \theta \in (0, 1) \}  \\
& \to 2 c_3 u_j \sgn(b^{*}_j).
\end{align*}
We now first consider the case that $c_3 = 0$. In this case,
we have
\[
 G_n(\u) - G_n(\0)  \overset{d}{\longrightarrow} \u^T \C \u - 2 \u^T  \z,
\]
which is convex w.r.t.\ $\u$. Then the minimizer of $\u^T \C \u {-} 2 \u^T  \z$ is $\u^{*}$ if and only if  $\C \u^{*} - \z=\0$. Since
$\hat{\u} \overset{d}{\rightarrow} \u^{*}$ (by epiconvergence), we obtain $ \sqrt{n}(\tilde{\b}_n - \b^{*})=\hat{\u}  \overset{d}{\rightarrow}
N(\0, \sigma^2 \C^{-1})$.

We then consider the case that $c_3 \in (0, \infty)$. Right now we have
\[
 G_n(\u) - G_n(\0)  \overset{d}{\longrightarrow} \u^T \C \u - 2 \u^T  \z+ 2c_3 \sum_{j \in \AM} u_j \sgn(b^{*}_j) +
2 c_3 \sum_{j \notin \AM}  |u_j|  \triangleq H_2(\u).
\]
$H_2(\u)$ is convex in $\u$. Let the minimizer of $H_2(\u)$ be $\u^{*}$. Then
\[
\C \u^{*} - \z + c_3 \s =0
\]
where $\s^T = (\sgn(\b^{*}_1)^T, \v^T)$ and $\v \in \RB^{p_2}$ with $\max_{j} |v_j| \leq 1$. Thus,
we have $\u^{*} \overset{d}{\rightarrow} N({\bf t}, \sigma^2 \Tha)$ where ${\bf t}=(t_1, \ldots, t_p)^T=-c_3 \C^{-1} \s$
and $\Tha=[\theta_{ij}]= \C^{-1}$.
For any $\epsilon>0$, when $n$ is significantly large and using Chebyshev's inequality,  we have that
\begin{align*}
\Pr\Big[|u_j^{*}|/\sqrt{n} \geq \epsilon \Big]
&= \Pr\Big[ |u_j^{*}| \geq \sqrt{n} \epsilon \Big] \\
& \leq
\Pr\Big[|u_j^{*} - t_j| \geq \sqrt{n} \epsilon - |t_j| \Big]
 \leq \frac{\sigma^2 \theta_{jj}}{(\sqrt{n} \epsilon - |t_j| )^2} \to 0
\end{align*}
for $j=1, \ldots, p$. Consequently,  $|u_j^{*}|/\sqrt{n} \overset{p}{\rightarrow}0$;
that is, $\tilde{\b}_n  \overset{p}{\rightarrow} \b^{*}$.

\section{The Proof of Theorem~\ref{thm:cc}}

\begin{proof} Since $\Phi(s)$ is a proper concave function in $s$ on $(0, \infty)$, we now compute its concave conjugate. That is,
\[
\min_{s>0} \{g(s) \triangleq w s -  \Phi(s)\}.
\]
Let the first-order derivative of $g(s)$ w.r.t.\ $s$ be equal to $0$, which yields
\[
 s = (\Phi')^{-1}(w).
\]
Thus, the corresponding minimum  (denoted $g^{*}$) is
\[
g^{*} =  {w}(\Phi')^{-1}(w) - \Phi((\Phi')^{-1}(w)).
\]
We denote $\varphi(z)= \Phi((\Phi')^{-1}(z))- z (\Phi')^{-1}(z)$. We now prove that $\varphi(z)$ satisfies
the conditions in Definition~\ref{def:22}. Since $\Phi(0)=0$ and $\Phi'(0)=1$, we have that $(\Phi')^{-1}(1)=0$
and $\Phi((\Phi')^{-1}(1))=0$. As a result, we have $\varphi(1)=0$.
The first-order and second-order derivatives of $\varphi(z)$ are
\[
\varphi'(z) =- (\Phi')^{-1}(z) \quad \mbox{ and }  \quad  \varphi{''}(z) = -\frac{1}{\Phi{''}((\Phi')^{-1}(z))}.
\]
We accordingly obtain that $\varphi'(1) =0$ and $ \varphi{''}(z)>0$ on $(0, \infty)$. Moreover, we have $\liml_{\z\to 0+} \varphi'(z) = -\infty$
in terms of Lemma~\ref{lem:lapexp} (which shows that $\liml_{u \to +\infty} \Phi'(u)=0$).
\end{proof}

\bibliographystyle{Chicago}
\bibliography{ncvs2}

\end{document}